\documentclass[11pt]{article}

\usepackage[preprint]{acl}

\usepackage{times}
\usepackage{bm}
\usepackage{latexsym}
\usepackage{subcaption}
\usepackage{placeins}
\usepackage[dvipsnames]{xcolor}
\usepackage{colortbl}     
\usepackage{makecell}
\usepackage{ulem}            
\definecolor{impGreen}{HTML}{2E7D32}
\definecolor{regRed}{HTML}{C62828} 
\usepackage[T1]{fontenc}

\usepackage[utf8]{inputenc}

\usepackage{microtype}

\usepackage{inconsolata}

\usepackage{graphicx}
\usepackage[most]{tcolorbox}
\usepackage{graphicx}
\usepackage{pifont}
\usepackage{ulem}
\usepackage{enumitem}
\usepackage{xspace}
\usepackage{bbding}
\usepackage{multirow, booktabs}
\usepackage{makecell}
\usepackage{amssymb}
\usepackage{amsmath}
\usepackage{amsthm}
\theoremstyle{definition}
\newtheorem{definition}{Definition}[section]
\usepackage{cleveref}
\usepackage{adjustbox}
\usepackage{listings}
\usepackage{float}
\usepackage{soul}
\usepackage{ulem}
\usepackage{twemojis}
\usepackage[table,svgnames]{xcolor} 
\usepackage{algorithm}
  \usepackage{algorithmic}
  \usepackage{xcolor}

\newtheorem{proposition}[definition]{Proposition}

\definecolor{green(pigment)}{rgb}{0.0, 0.65, 0.31}

\title{Entropy Is Not Enough: Unlocking Effective Reinforcement Learning for Visual Reasoning via Vision-Anchored Token Selection}

\author{
 \textbf{Senjie Jin\textsuperscript{1}\thanks{{ }\ Equal contribution.$^\dagger$Corresponding authors.}},
 \textbf{Peixin Wang\textsuperscript{1}$^{*}$},
 \textbf{Boyang Liu\textsuperscript{1}$^{*}$},
 \textbf{Xiaoran Fan\textsuperscript{1}$^{*}$},
\\
 \textbf{Shuo Li\textsuperscript{1}},
 \textbf{Zhiheng Xi\textsuperscript{1}},
 \textbf{Jiazheng Zhang\textsuperscript{1}},
 \textbf{Yuhao Zhou\textsuperscript{1}},
\\
 \textbf{Tao Gui\textsuperscript{1}$^\dagger$},
 \textbf{Qi Zhang\textsuperscript{1}},
 \textbf{Xuanjing Huang \textsuperscript{1}}
\\
 \textsuperscript{1}College of Computer Science and Artificial Intelligence,
Fudan University \\
 \texttt{\{sjjin24,peixinwang25,boyangliu25\}@m.fudan.edu.cn}, \texttt{tgui@fudan.edu.cn}\\
}

\begin{document}
\maketitle
\begin{abstract}
While token-level entropy is commonly recognized as effective for credit assignment in text-only reinforcement learning with verifiable rewards (RLVR), it remains unclear whether this mechanism still holds in visual reasoning. Our controlled study shows that this mechanism collapses in visual reasoning due to the omission of vision-sensitive tokens with naturally low entropy. 
Although existing multimodal RL methods increasingly acknowledge the importance of visual perception, they struggle to satisfy the inherent demand for interleaving precise perceptual grounding with semantic reasoning, either lacking systematic visual measurements or overlooking that token entropy primarily drives semantic exploration.
To address this, we introduce \textbf{VEPO} (Vision-Entropy token-selection for Policy Optimization), an effective RL framework explicitly integrating \emph{visual sensitivity} with \emph{token entropy} via a principled multiplicative coupling, where VEPO redirects gradient credit toward tokens which are simultaneously visually grounded and highly informative.
Extensive experiments demonstrate VEPO's leading performance, significantly outperforming the entropy-only baseline by 2.28 points at 7B-scale and 3.15 points at 3B-scale. Ablations further substantiate the soundness of our method\footnote{https://github.com/Leonnnnnn929/VEPO}.



\end{abstract}
\begin{figure}[t!]
    \centering
    \begin{subfigure}{\linewidth}
        \centering
        \makebox[\linewidth]{%
            \includegraphics[width=1.1\linewidth]{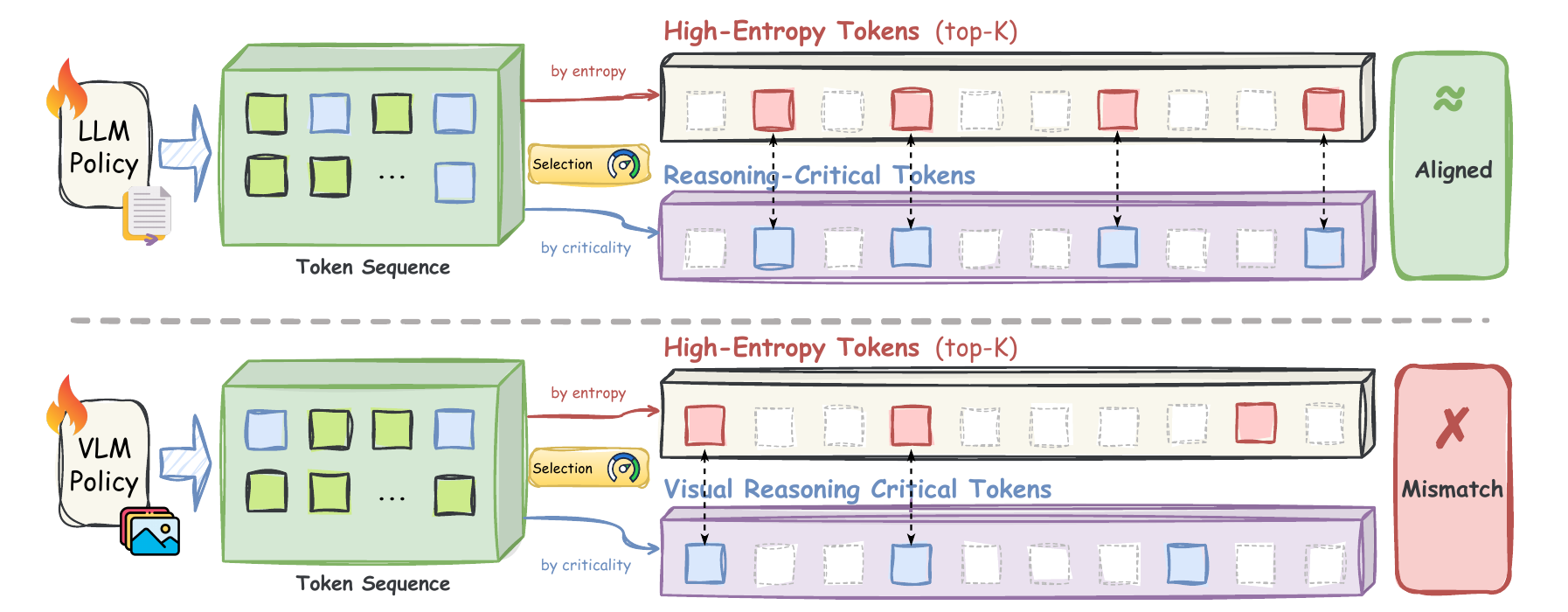}}
        \caption{Entropy selection succeeds in LLMs but collapses in VLMs, omitting vision-critical tokens with low entropy.}
        \label{fig:phenomenon figure}    
    \end{subfigure}
    \begin{subfigure}{\linewidth}
        \centering
        \makebox[\linewidth]{%
            \includegraphics[width=1.05\linewidth]{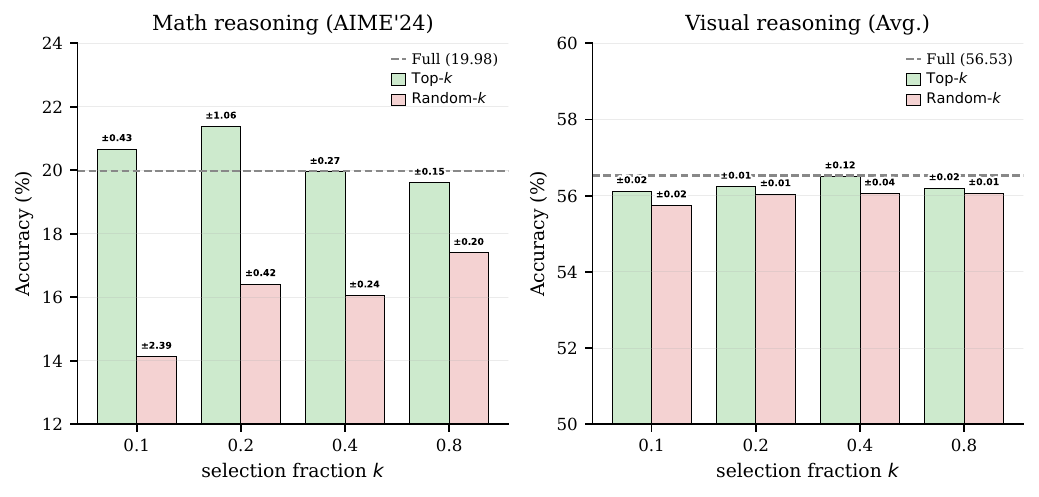}}
        \caption{Performance of different token-selection mechanisms on math/visual reasoning tasks.}
        \label{fig:llm/vlm-performance-gap}    
    \end{subfigure}
    \caption{The entropy-based token selection mechanism collapses on visual reasoning: A small subset of high-entropy token training outperforms \uline{Full} and substantially exceeds Random-$k$ on math reasoning. But on visual reasoning, entropy-based selection only marginally beats Random-$k$ and falls below \uline{Full}.}
  \label{fig:entropy-selection}
  \vspace{-1.2em}
\end{figure}

\vspace{-0.5em}
\section{Introduction}
 Reinforcement learning with verifiable rewards (RLVR) \citep{lambert2024tulu,shao2024deepseekmathpushinglimitsmathematical,yu2026dapo} has rapidly emerged as a pivotal paradigm for incentivizing complex reasoning abilities in Large Language Models (LLMs). A fundamental challenge is \emph{how to effectively allocate credit across individual tokens within long reasoning trajectories} \citep{zhong2024dpo,wang2026beyond}. Recent work in text-only RL converges on a compelling empirical regularity: a small fraction of high-entropy tokens, dubbed \emph{forking tokens}, dominate exploration and supply the majority of positive gradient signals \citep{he2026rethinking,wang2025stabilizing}. Token-level entropy $H_t$ thereby serves as a remarkably reliable proxy for credit assignment and policy improvement \citep{cui2025entropy,xi2026bapo,wang2025harnessing}.

While this entropy-driven mechanism proves highly effective for text-only reasoning tasks, whether it generalizes to visual reasoning remains a critical open question \citep{li2025roleentropyvisualgrounding}. We investigate this through a controlled study (\Cref{sec2:entropy mechanism collapses}) and observe a striking divergence: the entropy mechanism collapses once the visual modality is introduced (Figure~\ref{fig:llm/vlm-performance-gap}). On math reasoning (AIME'24), we reproduce the prior findings \citep{wang2026beyond}, yet the identical mechanism underperforms the full-token GRPO baseline on visual reasoning tasks, even yielding marginal gains over the random selection strategy.
We further probe the underlying cause and identify a previously uncharacterized phenomenon (\Cref{sec2:unpacking the collapse}): the pure entropy-based token selection inadequately captures genuinely vision-sensitive tokens (\Cref{fig:entropy-collapse-statistic}) that possess naturally low entropy \citep{favero2024multi,leng2024mitigating}. Quantitatively, at $k{=}0.2$, top-entropy selection recovers merely $60\%$ of tokens flagged by the vision-sensitive metrics, thereby we argue that token entropy alone remains insufficient for identifying the critical tokens in visual reasoning. Qualitative cases are deferred to~\Cref{fig:entropy-collapse-statistic}(c)(d) and~\Cref{sec:appendix:Qualitative Cases}. 

Unlike text-only reasoning tasks, visual reasoning intrinsically requires interleaving precise perceptual grounding with semantic reasoning \citep{DBLP:journals/corr/abs-2510-27492, DBLP:journals/corr/abs-2506-23918}. While recent work in multimodal RL \citep{liu2025noisyrolloutreinforcingvisualreasoning,PAPO,li2026rethinking} has increasingly recognized the importance of visual perception, they either lack systematic visual signal measurement (e.g., similarity, KL Divergence) \citep{li2026rethinking}, or neglect the fact that entropy \citep{SpotlightonToken, DBLP:journals/corr/abs-2604-01840} primary serve as a proxy of semantic exploration, which visual modalities cannot effectively resolve.


Motivated by above analysis, we propose \textbf{VEPO} (Vision-Entropy token-selection for Policy Optimization), an effective RL framework that explicitly couples \emph{vision-sensitivity} with \emph{token entropy} during token selection, redirecting gradient credit toward tokens that serve a true functional purpose in visual reasoning. 
Concretely, VEPO performs a counterfactual forward pass with a noise-perturbed image and quantifies token-level visual dependency by multiplicatively coupling two complementary signals: (i) the Jensen--Shannon divergence (JSD) \citep{menendez1997jensen} for distributional shift, and (ii) the absolute entropy gap $|\Delta H_t|$ for direction-agnostic shifts.
This joint formulation is conceptually inspired by  the \emph{aleatoric--epistemic} uncertainty decomposition \citep{houlsby2011bayesian,depeweg2018decomposition,kendall2017uncertainties}, establishing distributional divergence and predictive confidence as complementary axes.
Furthermore, the multiplicative coupling acts as a soft probabilistic OR \citep{xu2026tip, pearl2014probabilistic}, guaranteeing the comprehensive capture of visually sensitive signals by preventing the omission of tokens that spike along either axis (\Cref{appendix:theory}).
Ultimately, by modulating this fused signal with token entropy, VEPO ensures policy updates are driven exclusively by \emph{visual forking tokens} that are simultaneously visually grounded and highly informative.




We evaluate VEPO with Qwen2.5-VL-7B/3B-Instruct \citep{DBLP:journals/corr/abs-2502-13923}. Experimental results demonstrate that VEPO achieves 
leading performance at the 7B scale, outperforming the top-entropy baseline by $+2.28$, with consistent gains extending to the 3B scale. Moreover, we conduct comprehensive ablation studies and analysis (\Cref{sec5:Ablation Study and Analysis}) to validate the soundness and effectiveness of our proposed method.

\begin{figure*}[t!]
    \includegraphics[width=1.01\linewidth]{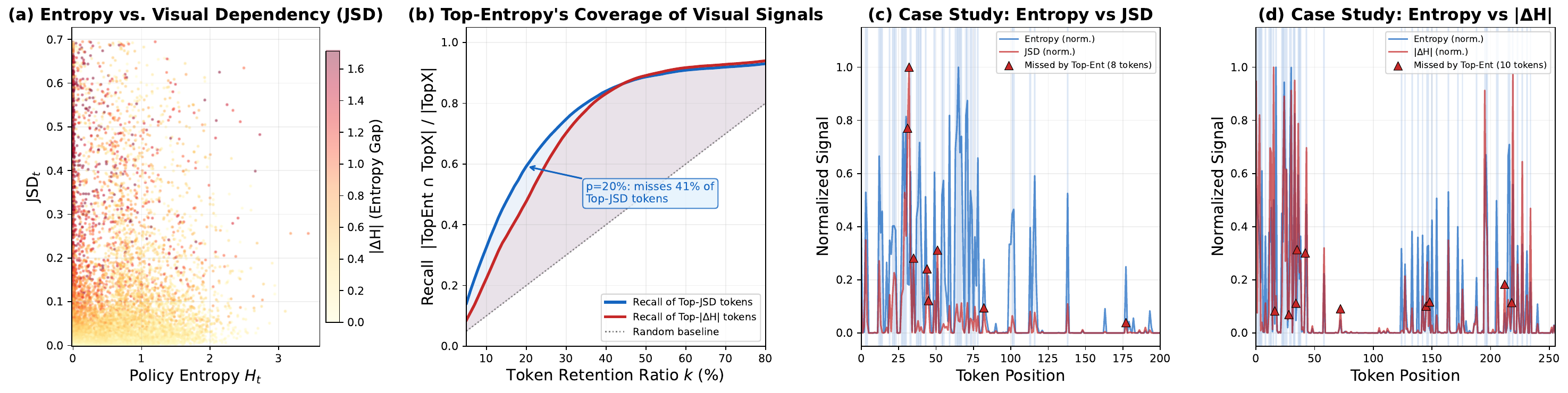}
    \centering
 	\caption{(a) Many high-JSD / high-$|\Delta H|$ tokens lie in the low-entropy region. (b) Top-Entropy selection misses $41\%$ of Top-JSD tokens at $k\!=\!20\%$, alongside a comparable proportion of Top-$|\Delta H|$ tokens. (c,d) Tokens flagged by JSD / $|\Delta H|$ (red triangles) frequently fall in entropy valleys and are missed by Top-Entropy selection.}\label{fig:entropy-collapse-statistic}
  \vspace{-0.5cm}
\end{figure*}

Our contributions are summarized as follows:
\begin{itemize}
    \item We provide a comprehensive diagnosis revealing why the top-entropy token selection mechanism collapses in visual reasoning, attributing to its inherent omission of visually grounded tokens.
    \vspace{-0.3em}
    \item We propose \textbf{VEPO}, a novel RL framework that integrates visual sensitivity and token entropy signals to achieve a visually grounded and highly informative optimization.
    \vspace{-0.3em}
    \item Extensive experiments across seven benchmarks and two model scales, alongside comprehensive ablation studies, demonstrate the effectiveness and soundness of VEPO.
\end{itemize}

\section{Preliminary Experiments}\label{sec2:Preliminary Experiments}
Recent works on math reasoning have uncovered a striking empirical observation:
training on a small subset of high-entropy tokens yields non-trivial benefits \citep{wang2026beyond,he2026rethinking, wang2025stabilizing}. Building upon this, we reproduce the experiments and further investigate whether this mechanism generalizes to visual reasoning tasks. Our empirical findings and analyses are elaborated in the subsequent sections, while detailed experimental setup is deferred to~\Cref{appendix:Preliminary Experiments Setup}.

\subsection{The entropy mechanism collapses on visual reasoning}
\label{sec2:entropy mechanism collapses}
The entropy-based token-selection mechanism collapses when transferred from text to visual reasoning (\Cref{fig:entropy-selection}). On math reasoning (AIME'24), selecting a small fraction of high-entropy tokens (Top-10\%/20\%) exceeds the \uline{Full}-token baseline, while substantially outperforming Random-$k$ selection across all fractions. This empirically verifies the conjecture presented in \citep{wang2026beyond}. In contrast, on visual reasoning, we find: 1) Top-$k$ entropy selection consistently falls below the \uline{Full}-token baseline and yields only marginal gains over Random-$k$ selection. 2) Both the absolute performance and the Top/Random selection gaps remain consistent regardless of the fractions $k$. These results indicate that token entropy fails to precisely identify critical reasoning decision points once visual evidence is introduced. Detailed results are deferred to~\Cref{appendix:Preliminary Details}.

\subsection{Unpacking the collapse: The deficiency of visual signals}\label{sec2:unpacking the collapse}
Drawing inspiration from recent studies that highlight visual evidence during reasoning \citep{LookAgain, SpotlightonToken}, we hypothesize that high-entropy tokens are sometimes misaligned with critical, visually-grounded decision points. 
To probe this, we measure two token-level visual dependence signals: the Jensen--Shannon divergence (JSD) \citep{menendez1997jensen} and the entropy gap $|\Delta H_t|$ (Detailed definitions are provided in~\Cref{sec3:token-level visual sentivity})  between the model's predictive distribution with and without the image noise.

We get several intriguing observations (illustrated in \Cref{fig:entropy-collapse-statistic}): First, a substantial fraction of high-JSD and high-$|\Delta H|$ tokens concentrates in the low-entropy region (\Cref{fig:entropy-collapse-statistic}(a)). Mechanistically, visual evidence rules out distractor continuations and sharpens the posterior, so visually-grounded tokens sometimes become more confident, manifesting as lowered entropy \citep{favero2024multi, leng2024mitigating}. Meanwhile, high-entropy tokens are driven more by linguistic indeterminacy \citep{bigelow2025forking, park2026reasoning}. Second, \Cref{fig:entropy-collapse-statistic}(b) quantifies the representational gap: at retention ratio $k=0.2$, Top-Entropy selection recovers a mere $59\%$ of Top-JSD tokens, and similarly, a comparable proportion of Top-$|\Delta H|$ tokens. As $k$ grows, this recall scales in lockstep with the accuracy curve shown in \Cref{fig:entropy-selection}(b), gradually approaching the full-token baseline. Performance peaks at $k=0.4$, where coverage is sufficient and noise remains bounded. At $k=0.8$, the marginal gain vanishes as the selection is diluted with linguistically-uncertain but visually-irrelevant tokens. \Cref{fig:entropy-collapse-statistic}(c)(d) are two qualitative cases demonstrate how purely entropy-based selection inadequately captures vision-sensitive tokens.

\begin{figure*}[t!]
    \includegraphics[width=1.0\linewidth]{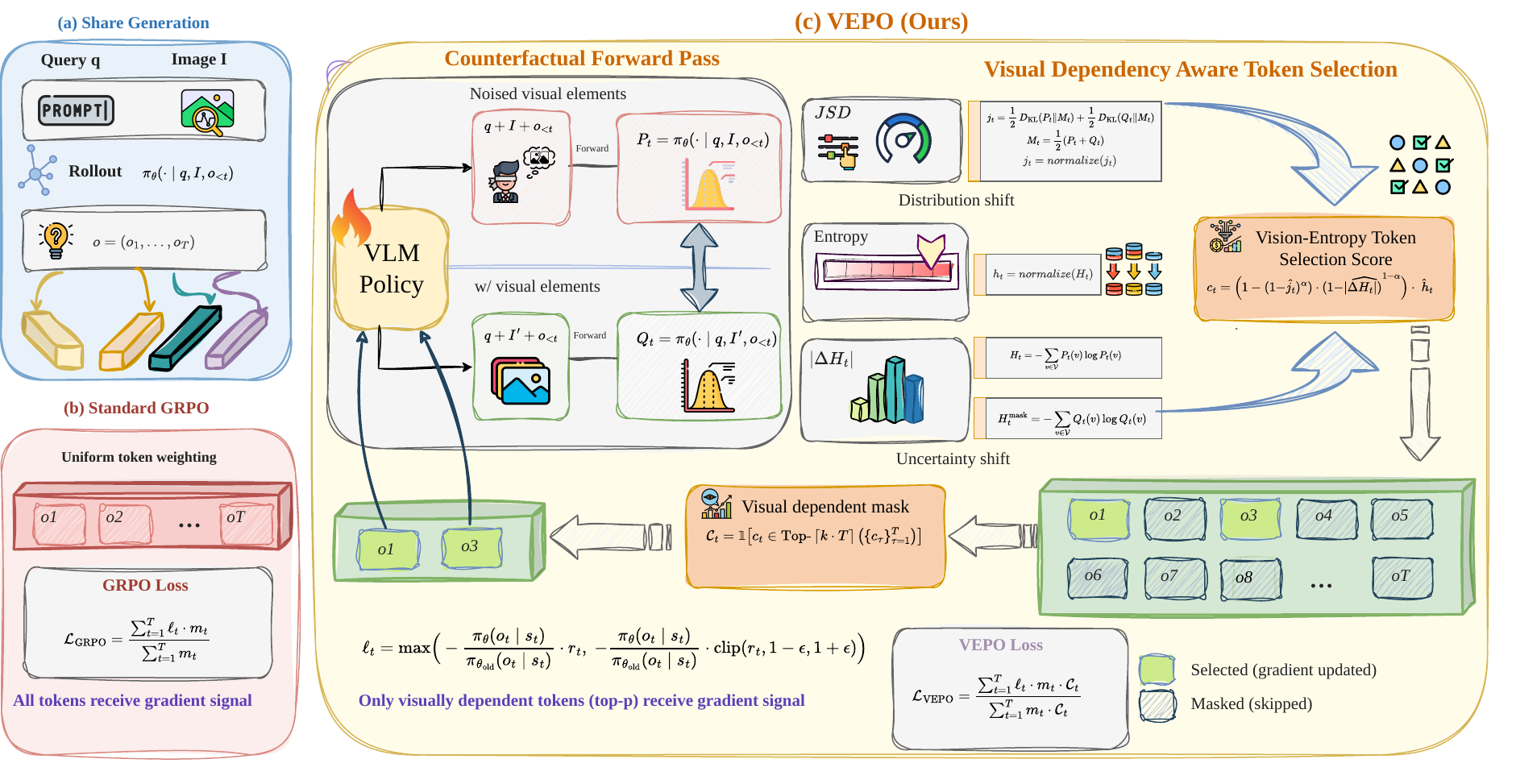}
    \centering
    \vspace{-2em}
 	\caption{The main training pipeline. VEPO performs a counterfactual forward pass with a noise perturbed image and quantifies token-level visual dependency by multiplicatively coupling two complementary signals: \textbf{JSD} and $\bm{|\Delta H_t|}$, which are then modulated with token \textbf{entropy} for simultaneously visually grounded and informative optimization. The detailed algorithm is in \Cref{appendix:Algorithm}. }\label{fig:main}
  \vspace{-0.5cm}
\end{figure*}

\section{Method: VEPO}
\paragraph{Motivation.}
From our findings and analysis in \Cref{sec2:Preliminary Experiments}, we argue that token entropy alone is insufficient to identify critical \textit{forking tokens} in visual reasoning. Following recent works that emphasize visual evidence in the reasoning process \citep{liu2025noisyrolloutreinforcingvisualreasoning, Look-Back, DoMLLMsReallySeeIt,ReinforcedAttentionLearning}, we aim to highlight the role of visual signals in evaluating token uncertainty. Hence, we propose \textbf{VEPO}, an effective RL training algorithm driven by a joint visual sensitivity and entropy token-selection mechanism. We elaborate on the metrics used to measure token visual sensitivity in \Cref{sec3:token-level visual sentivity}, followed by our selection mechanism and training method in \Cref{sec3:Joint Selection} and \Cref{sec3:training methods}, all of which are illustrated in \Cref{fig:main}.


\subsection{Token-Level Visual Sensitivity}\label{sec3:token-level visual sentivity}
Given a vision-language policy model parameterized by $\pi_\theta$, as shown in~\Cref{fig:main}, we perform two forward passes for each generated response $o = (o_1, \ldots, o_T)$,  
yielding a pair of  symmetric full-vocabulary next-token distributions at each position $t$:
 \begin{equation}
\begin{aligned}
    P_t(\cdot) &= \pi_\theta(\cdot \mid q, I, o_{<t}), \\
    Q_t(\cdot) &= \pi_\theta(\cdot \mid q, I', o_{<t}),
\end{aligned}
\label{eq:forward}
\end{equation}
where $P_t$ and $Q_t$ denote the token distributions conditioned on the original image $I$ and the noisy image $I'$, respectively, with the same response prefix $o_{<t}$. Their discrepancy therefore isolates the effect of visual corruption. To quantify this distributional change, we use the Jensen--Shannon divergence (JSD) \citep{menendez1997jensen}:
\begin{equation}
\mathrm{JSD}_t = \frac{1}{2}\,D_{\mathrm{KL}}(P_t \| M_t)
               + \frac{1}{2}\,D_{\mathrm{KL}}(Q_t \| M_t),
\label{eq:jsd}
\end{equation}
where
\begin{equation}
M_t = \frac{1}{2}(P_t + Q_t).
\end{equation}
Detailed comparisons and analysis of JSD and KL-based variants are provided in ~\cref{Analysis:visual signal variants}.

We further compute the entropy gap of the two distributions to characterize the model's token predictive uncertainty: 
\begin{align}
      H_t         &= -\sum_{v \in \mathcal{V}} P_t(v) \log P_t(v),
      \label{eq:entropy_normal} \\
      H_t^{\mathrm{mask}} &= -\sum_{v \in \mathcal{V}} Q_t(v) \log Q_t(v),
      \label{eq:entropy_cf}
  \end{align}
where $H_t$ and $H_t^{\mathrm{mask}}$ denote the entropy under the original image and the corrupted image. We define the entropy gap as
\begin{equation}
\Delta H_t = H_t^{\mathrm{mask}} - H_t,
\label{eq:entropy_gap}
\end{equation}
whose magnitude $|\Delta H_t|$ measures how strongly visual corruption changes predictive uncertainty at position $t$. We therefore use $|\Delta H_t|$ as a direction-agnostic signal of visual sensitivity. Further analysis is provided in ~\cref{Analysis:visual signal variants}.

\subsection{Joint Vision-Entropy for Token Selection}
\label{sec3:Joint Selection}
We define a token-level vision-entropy score to select both visually dependent and highly informative tokens. Following prior token-importance formulations, we combine two visual signals with a probabilistic OR \citep{xu2026tip, pearl2014probabilistic}.

\begin{definition}[\textbf{Vision-Entropy Token Selection Score}]
Let $\hat{j}_t \in [0,1]$, $|\widehat{\Delta H}_t| \in [0,1]$, and $\hat{h}_t \in [0,1]$ denote the per-response min-max normalized JSD, absolute entropy gap, and entropy at position $t$. We define the score of token $t$ as
\begin{equation}
c_t :=
\underbrace{\Big(1 - (1-\hat{j}_t)^{\alpha}(1-|\widehat{\Delta H}_t|)^{1-\alpha}\Big)}_{g_t}
\cdot \hat{h}_t,
\label{eq:score_d}
\end{equation}
where $\alpha \in [0,1]$ controls the trade-off between JSD and $|\Delta H_t|$.
\end{definition}

Here $g_t$ is a joint visual-dependency score that combines two complementary signals, $\hat{j}_t$ and $|\widehat{\Delta H}_t|$. By~\cref{eq:score_d}, it increases with either signal and peaks when both are strong. The entropy term $\hat{h}_t$ then acts as a modulation factor, constrained by the visual signal, ensuring that we prioritize tokens that are both visually sensitive and high informative. We provide an information-theoretic interpretation of this formulation in Appendix~\ref{appendix:theory}. Detailed ablations are provided in~\Cref{sec5:ablation:Visual-Entropy Token Selection Score}. 

As shown in Figure~\ref{fig:main}, we then select the top-$k$ fraction of tokens within each response. For a response of length $T$, we define the binary mask:
\begin{equation}
\mathcal{C}_t = 1\!\left[c_t \in \mathrm{Top}\text{-}\left\lceil kT \right\rceil \left(\{c_\tau\}_{\tau=1}^{T}\right)\right],
\label{eq:mask}
\end{equation}
where only tokens with $\mathcal{C}_t=1$ are used for the policy gradient update.

\subsection{Reinforcement Enhanced Visual Reasoning}
\label{sec3:training methods}

Standard GRPO \citep{shao2024deepseekmathpushinglimitsmathematical} computes the policy gradient loss over all valid response tokens. We modify it by restricting the update to the selected subset indicated by $\mathcal{C}_t$:
\begin{equation}
\mathcal{L}_{\mathrm{VEPO}} =
\frac{\sum_{t=1}^{T} \ell_t \cdot m_t \cdot \mathcal{C}_t}
{\sum_{t=1}^{T} m_t \cdot \mathcal{C}_t},
\label{eq:vipo_loss}
\end{equation}
where $m_t$ is the standard response mask. The per-token clipped policy gradient loss is:
\begin{equation}
\begin{aligned}
\ell_t &= \max\!\big({-}A \cdot r_t,\; {-}A \cdot \hat{r}_t\big), \\
\hat{r}_t &= \mathrm{clip}(r_t,\, 1{-}\epsilon,\, 1{+}\epsilon),
\end{aligned}
\label{eq:per_token_loss}
\end{equation}
with importance ratio:
\begin{equation}
r_t = \frac{\pi_\theta(o_t \mid s_t)}{\pi_{\theta_{\mathrm{old}}}(o_t \mid s_t)}.
\label{eq:importance_ratio}
\end{equation}
By replacing $m_t$ with $m_t \cdot \mathcal{C}_t$, VEPO concentrates gradient updates on \emph{visual forking} tokens.

\begin{table*}[t!]
\centering
\setlength{\tabcolsep}{5.5pt}
\renewcommand{\arraystretch}{1.2}
\resizebox{\linewidth}{!}{
\begin{tabular}{l|cccccccc}
\Xhline{ 1.2pt}
\rowcolor{CadetBlue!20}
\textbf{Model / Dataset} & \textbf{Geo3K} & \textbf{MMK12} & \textbf{HalluBench} & \textbf{MathVista} & \textbf{We-Math} & \textbf{MathVerse} & \textbf{MathVision} & \textbf{Avg.} \\
\Xhline{1.2pt}

\multicolumn{9}{c}{\textit{Closed-source Models}} \\
\hline
\rowcolor{gray!10}
GPT-4o \citep{openai2024gpt4ocard}              & -- & 49.90 & 71.40 & 63.80 & 69.00 & 50.80 & 30.40 & -- \\ 
\rowcolor{gray!10}
o1 \citep{openai2026openaio1card}               & -  - & --   & --   & 73.90 & 98.70 & 57.00 & 60.30 & -- \\
\rowcolor{gray!10}
Claude-3.7-Sonnet \citep{claude37}              & -- & 55.30 & --   & 66.80 & 72.60 & 52.00 & 41.30 & -- \\
\rowcolor{gray!10}
Gemini-2.0-Flash \citep{geminiteam2025gemini25} & -- & 65.20 & --   & 70.40 & 71.40 & 59.30 & 41.30 & -- \\
\hline

\multicolumn{9}{c}{\textit{Open-source Models}} \\
\hline
\rowcolor{gray!10}
Qwen2.5-VL-32B \citep{DBLP:journals/corr/abs-2502-13923}  & -- & 61.50 & -- & 71.40 & 65.40 & 48.20 & 35.90 & -- \\
\rowcolor{gray!10}
InternVL2.5-38B \citep{chen2025internvl}               & -- & 66.80 & -- & 74.70 & 69.10 & 49.90 & 40.10 & -- \\
\rowcolor{gray!10}
Qwen2.5-VL-72B \citep{DBLP:journals/corr/abs-2502-13923}  & -- & 58.00 & -- & 71.90 & 67.50 & 49.40 & 31.80 & -- \\
\rowcolor{gray!10}
QvQ-72B-Preview \citep{qvq-72b-preview}                & -- & 70.50 & -- & 74.80 & 72.40 & 57.60 & 38.10 & -- \\
\rowcolor{gray!10}
InternVL2.5-78B \citep{chen2025internvl}               & -- & 61.60 & -- & 72.30 & 66.30 & 51.70 & 32.20 & -- \\
\hline

\multicolumn{9}{c}{\textit{Visual-focused RL Fine-tuning Methods (7B)}} \\
\hline
\rowcolor{gray!10}
NoisyRollout-7B \citep{liu2025noisyrolloutreinforcingvisualreasoning} & 48.42 & 66.05 & 71.29 & 71.70 & 69.20 & \textbf{49.16} & \uline{28.91} & 57.82 \\
\rowcolor{gray!10}
PAPO-DAPO-7B \citep{PAPO}                                                 & 50.08 & 66.05 & \textbf{72.24} & \textbf{73.10} & 68.69 & 48.07 & 26.23 &  57.78 \\
\rowcolor{gray!10}
VPPO-7B \citep{SpotlightonToken}                                       & \uline{50.42} & \uline{69.02} & 70.14 & \uline{72.10} & 69.43 & 48.21 & 28.19 & \uline{58.22} \\
\rowcolor{gray!10}
R1-ShareVL \citep{yao2026rsharevl}                                     & 46.59 & 65.68 & 69.93 & 68.10 & \textbf{70.17} & 47.39 & 27.27 & 56.45 \\
\hline

\multicolumn{9}{c}{\textit{Our Method}} \\
\hline
\rowcolor{gray!10}
Qwen2.5-VL-7B-Instruct \citep{DBLP:journals/corr/abs-2502-13923} & 35.94 & 53.66 & 63.51 & 66.87 & 62.13 & 43.27 & 25.10 & 50.07 \\
\rowcolor{gray!10}

\hspace{2em}+ GRPO                                              & 49.42 & 65.18 & 70.45 & 71.60 & 67.18 & 48.53 & 27.40 & 57.11 \\
\rowcolor{gray!10}
\hspace{2em}+ GRPO (Top 20\% high entropy)
& 47.59 & 64.93 & 68.56 & 70.40 & 67.13 & 48.25 & \textbf{28.95} & 56.54 \\
\rowcolor{gray!10}
\hspace{2em}+ GRPO (Top 40\% high entropy)
& 46.76 & 63.44 & 69.09 & 71.40 & 68.45 & 48.86 & 27.53 & 56.50 \\
\hline
\rowcolor{gray!10}
\hspace{2em}+ \textbf{VEPO}
& \textbf{51.58} & \textbf{69.64} & \uline{71.71} & 72.00 & \uline{69.54} & \uline{48.93} & 28.31 & \textbf{58.82} \\[-1.9pt]
\rowcolor{gray!10}
\hspace{2em}${\triangle}$  vs. GRPO(Top 20\% high entropy)
 & {\textcolor{green(pigment)}{(+3.99)}}
 & {\textcolor{green(pigment)}{(+4.71)}}
 & {\textcolor{green(pigment)}{(+3.15)}}
 & {\textcolor{green(pigment)}{(+1.60)}}
 & {\textcolor{green(pigment)}{(+2.41)}}
 & {\textcolor{green(pigment)}{(+0.68)}}
 & {\textcolor{regRed}{($-$0.64)}}
 & {\textcolor{green(pigment)}{(+2.28)}} \\
\Xhline{1.2pt}
\end{tabular}
}
\vspace{-0.3em}
\caption{The main experimental results across multiple visual math and hallucination benchmarks. Models are trained on the Geo3K and MMK12 datasets, which serve as in-domain evaluations, while the remaining datasets assess out-of-domain generalization. Our proposed VEPO (bottom) demonstrates superior overall performance compared to the corresponding baselines. Best results are highlighted in \textbf{bold}, and the second-best are \uline{underlined}.}
\label{tab:performance_simplified_7B}
\vspace{-1.1em}
\end{table*}

\section{Experiment}
\subsection{Experiment Setup} 
\paragraph{Dataset and Evaluation.}
We construct the training set from subsets of Geometry3K \citep{geo3k} and MMK12 \citep{meng2025mm}, resulting in approximately 4.2K training samples. We evaluate the model on both in-domain and out-of-domain benchmarks: the Geometry3K and MMK12 validation splits serve as in-domain evaluation, while a set of public visual reasoning benchmarks is used to assess out-of-domain generalization. Detailed benchmark descriptions are provided in Appendix~\ref{appendix:implementation details}.

\paragraph{Implementation.}We use the EasyR1 framework\cite{zheng2025easyr1} as the reinforcement learning framework for vision-language model training. The policy models are initialized from Qwen2.5-VL-3/7B-Instruct\footnote{\label{fn:qwen}\scriptsize\url{https://huggingface.co/collections/Qwen/qwen25-vl}}, followed by fine-tuning on our custom dataset. More implementation details can be found in Appendix~\ref{appendix:implementation details}.  
\paragraph{Baseline.}
We compare our method with several baselines: closed-source multimodal models, open-source large-scale VLMs, visual-focused RL fine-tuning methods, and top-entropy token selection methods. For RL-based and entropy-based selection baselines, we retrain these methods under the same 4.2K mixed training set for fair comparison. More details can be found in Appendix~\ref{appendix:Baseline}.

\subsection{Main Results}
The main results are presented in \Cref{tab:performance_simplified_7B}.
We primarily analyze the model performance compared with top-entropy selection methods and effective visual-focused RL fine-tuning methods. We also included the scaling experiments on the 3B model.

\paragraph{Comparison with Top-Entropy selection.}
We first compare VEPO against entropy-based methods. As shown at the bottom of~\Cref{tab:performance_simplified_7B}, both entropy-only variants (top-20\% / top-40\%) underperform the full-token GRPO baseline on average (56.54 / 56.50 vs.\ 57.11), corroborating our preliminary observation in~\Cref{sec2:Preliminary Experiments} that token entropy alone is an insufficient selection signal for visual reasoning tasks. By further incorporating the visual-sensitivity criterion, VEPO reaches a \textbf{+2.28} gain over the strongest entropy-only variant and outperforms full-token GRPO by \textbf{+1.71}. The benefits are more notable on in-domain benchmarks (Geo3K \textbf{+3.99}, MMK12 \textbf{+4.71}) and transfer consistently to all out-of-domain tasks except for a marginal -0.64 drop on MathVision. These results suggest that visual sensitivity serves as an effective complementary signal to token entropy, helping to identify the effective decision points in visual reasoning.

\paragraph{Comparison with Visual-focused RL methods.}
We further compare VEPO with representative visual-focused RL fine-tuning methods. As shown in the middle of~\Cref{tab:performance_simplified_7B}, VEPO attains the highest average of \textbf{58.82}, consistently surpassing all baselines. Among these, VPPO \citep{SpotlightonToken} is the most competitive one, leveraging the KL divergence between clear and perturbed images for token selection. While we acknowledge the crucial role of visual signals in visual reasoning, the reasoning trajectory itself remains predominantly textual. Broad exploration \citep{cheng2026reasoning} along this textual reasoning chain is essential for pushing beyond performance ceiling, token entropy therefore remains an indispensable signal. This motivation lies at the heart of VEPO, which couples visual sensitivity with token entropy through a multiplicative scoring mechanism. A more comprehensive analysis is deferred to~\Cref{sec5:ablation:Visual-Entropy Token Selection Score}. 

\paragraph{Results across Different Model Scales.}
To verify VEPO's scalability, we evaluate it on Qwen2.5-VL-3B-Instruct (\Cref{tab:performance_simplified_3B}), yielding consistent conclusions. Details are deferred to~\Cref{appendix:Results on Qwen2.5-VL-3B-Instruct}.
Furthermore, VEPO-7B matches or exceeds several open-source multimodal models that are 4--10$\times$ larger. On MMK12, VEPO (\textbf{69.6}) significantly outperforms 38B--78B variants of Qwen2.5-VL and InternVL2.5. Similarly, on MathVista, it overtakes Qwen2.5-VL-72B and remains highly competitive with InternVL2.5-78B. These results further corroborate the effectiveness of our VEPO.

\begin{table*}[!t]
\centering
\tiny
\setlength{\tabcolsep}{6pt}
\renewcommand{\arraystretch}{1.2}
\resizebox{\linewidth}{!}{
\begin{tabular}{l|cccccccc}
\Xhline{1.2pt}
\rowcolor{CadetBlue!20}
\textbf{Variant} & \textbf{Geo3K} & \textbf{MMK12} & \textbf{HalluBench} & \textbf{MathVista} & \textbf{We-Math} & \textbf{MathVerse} & \textbf{MathVision} & \textbf{Avg.} \\
\Xhline{1.2pt}

\rowcolor{gray!10}
\textbf{Full VEPO} & 51.58 & 69.64 & 71.71 & 72.00 & 69.54 & 48.93 & 28.31 & 58.82 \\
\hline

\multicolumn{9}{c}{\textit{Ablation on VEPO token selection signals}} \\
\hline
\rowcolor{gray!10}
\hspace{2em}w/o JSD & 46.26 & 66.67 & 69.40 & 68.70 & 66.84 & 47.82 & 27.30 & 56.14 \\
\rowcolor{gray!10}
\hspace{2em}w/o Entropy Gap & 46.42 & 65.43 & 68.35 & 70.80 & 66.67 & 48.32 & 26.78 & 56.11 \\
\rowcolor{gray!10}
\hspace{2em}w/o Entropy & 47.75 & 69.02 & 70.35 & 72.30 & 69.48 & 48.98 & 27.34 & 57.89 \\
\hline

\multicolumn{9}{c}{\textit{Ablation and Analysis on variants of VEPO token selection mechanism}} \\
\hline
\rowcolor{gray!10}
\hspace{2em}Replacing JSD with KL & 50.42 & 69.27 & 69.30 & 71.60 & 68.22 & 49.87 & 27.73 & 58.06 \\
\rowcolor{gray!10}
\hspace{2em}Bottom 80\% & 41.60 & 60.97 & 68.45 & 66.40 & 67.53 & 43.89 & 26.38 & 53.60 \\
\rowcolor{gray!10}
\hspace{2em}$\max(\Delta H_t, 0)$ & 50.42 & 67.78 & 69.71 & 71.40 & 68.67 & 49.74 & 28.42 & 58.02 \\
\rowcolor{gray!10}
\hspace{2em}Additive Fusion & 49.75 & 70.01 & 68.66 & 71.60 & 67.99 & 50.23 & 28.68 & 58.13 \\
\Xhline{1.2pt}
\end{tabular}
}
\vspace{-0.3em}
\caption{Ablation and analysis of VEPO token selection variants. \textbf{Full VEPO} denotes the complete joint selection. The upper block ablates each signal individually (JSD, entropy gap, or entropy), while the lower block examines alternative formulations of the VEPO's selection mechanism. \textbf{Additive Fusion} replaces the multiplicative aggregation of two visual signals with linear interpolation: $g_t = \alpha\hat{j}_t + (1\!-\!\alpha)|\widehat{\Delta H_t}|$.}
\label{tab:ablation_components}
\vspace{-1em}
\end{table*}

\section{Ablation Study and Analysis}\label{sec5:Ablation Study and Analysis}
\subsection{Ablation Study} \label{sec5:ablation:Visual-Entropy Token Selection Score}
\paragraph{Signals of VEPO Token Selection Mechanism.} 
We first ablate the three signals in our joint selection mechanism (\Cref{eq:score_d}). As shown in~\Cref{tab:ablation_components}, removing any signal degrades average performance, confirming their complementary roles. The two visual signals (JSD and $|\Delta H_t|$) contribute the most, with their removal causing nearly identical drops of $-2.68$ and $-2.71$, respectively. Although omitting the entropy signal $\hat{h}_t$ yields a milder average decline ($-0.93$), we find it severely impacts benchmarks dominated by long textual reasoning chains (Geo3K, MathVision) while remaining neutral on perception-heavy ones (MathVerse, MathVista). Empirically validating that entropy and visual sensitivity are \emph{complementary signals} rather than competing alternatives.


\paragraph{Visual Sensitivity $\alpha$ and Token Fraction $k$.}
\label{sec5.1:ablation on hyperpara}
We further investigate the effects of two key hyperparameters:
the visual sensitivity balancing coefficient $\alpha$ and the token retention ratio $k$. 
As illustrated in ~\Cref{fig:ablation_alpha_k}, we summarize several 
key findings: 
1) The optimal configuration is $(k{=}0.2, \alpha{=}0.7)$, yielding the highest average score. 
2) The dense retention ratio $k{=}0.4$ uniformly underperforms the sparser settings across all $\alpha$, confirming that retaining too many tokens introduces weakly grounded positions into the policy update, consistent with \citep{wang2026beyond}. 
3) Performance peaks at $\alpha{=}0.7$ and degrades sharply at either extreme ($\alpha{=}0.1$ or $\alpha{=}0.9$), empirically confirming that JSD and $|\Delta H_t|$ serve as \emph{complementary visual signals}, with JSD playing a more prominent role.
4) Interestingly, the optimal $k$ adapts to the strength of $\alpha$. When $\alpha$ is small to moderate ($0.1\sim0.5$), $k{=}0.3$ outperforms $k{=}0.2$, as a weaker JSD signal benefits from a broader token pool.
Conversely, at $\alpha{\geq}0.7$, $k{=}0.2$ becomes optimal as the stronger JSD signal enables reliable selection from a sparser set.
Detailed results on each benchmark are deferred to~\Cref{appendix:Ablation a and k}.

\begin{figure}[t!]
    \centering
    \includegraphics[width=0.95\linewidth]{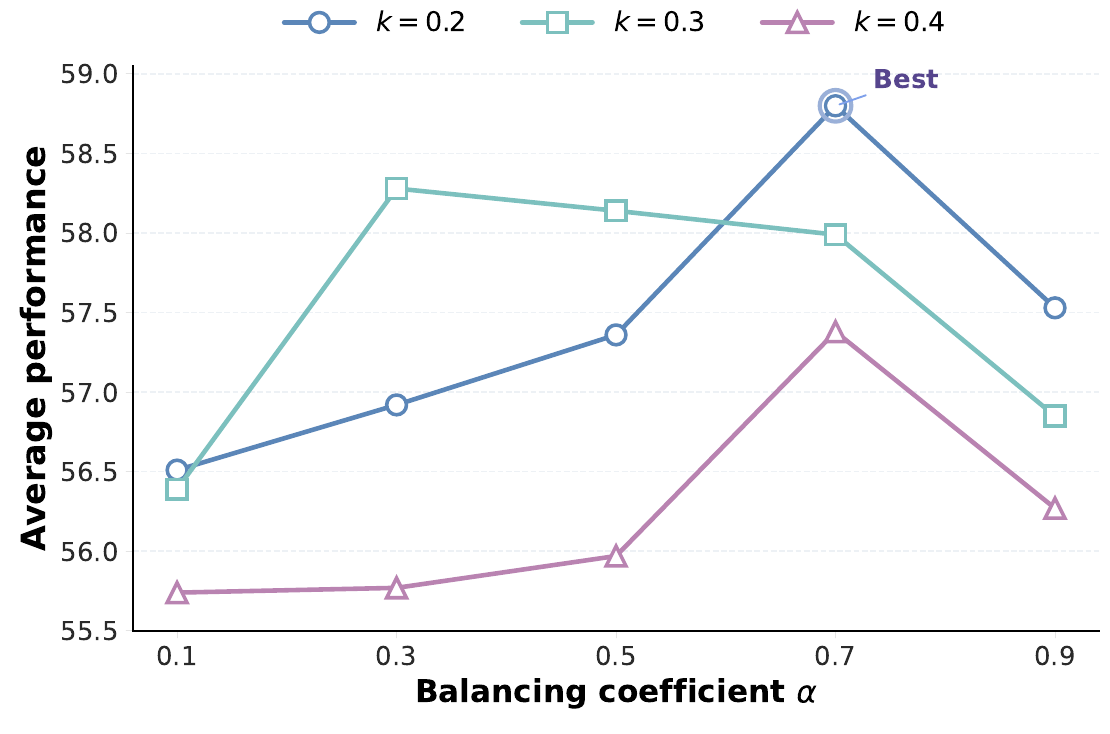}
    \vspace{-0.5em}
    \caption{Ablations on the balancing coefficient $\alpha$ and the token selection ratio $k$.}
    \label{fig:ablation_alpha_k}
    \vspace{-1em}
\end{figure}

\begin{figure*}[t!]
    \includegraphics[width=1.01\linewidth]{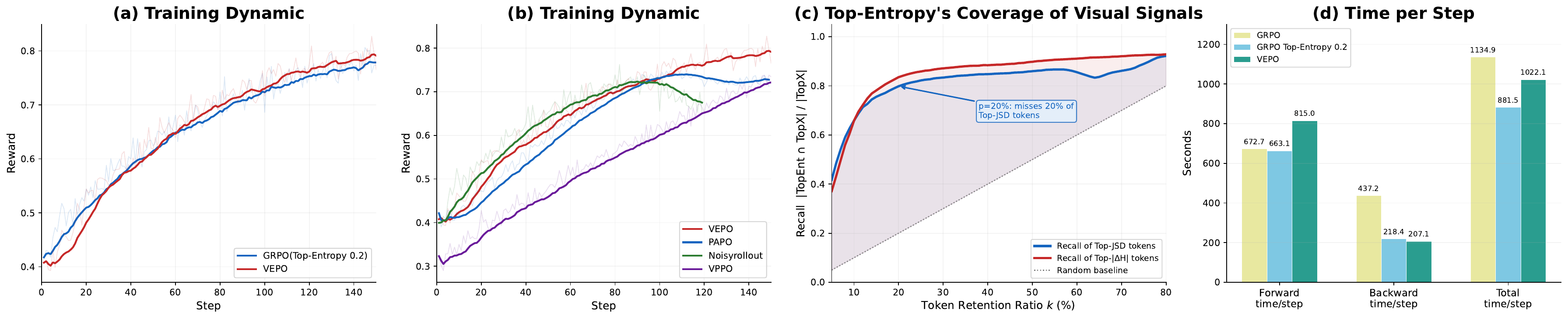}
    \centering
 	\caption{(a)(b) The training dynamics of VEPO and top-entropy and visual-focused RL methods (VPPO \citep{SpotlightonToken} removes format reward, we follow the default setting). (c) At $k{=}20\%$, VEPO covers more vision-sensitivity tokens, reducing miss rate to $20\%$. (d) Per-step wall-clock time of VEPO and GRPO baselines.}\label{fig:figure2}
  \vspace{-0.5cm}
\end{figure*}
\subsection{Analysis}
\paragraph{Comparison with visual signal variants.}
~\label{Analysis:visual signal variants}
We compare our visual signal formulation against three principled alternatives. (1) Replacing JSD with KL, the metric utilized in VPPO \citep{SpotlightonToken}, decreases the average by $0.76$ points. Unlike VPPO's trajectory-level advantage reweighting, VEPO employs token-level selection driven by a multiplicative coupling $v_t = g_t \cdot \hat{h}_t$. This makes unbounded signals like KL highly susceptible to outlier tokens, whereas JSD remains bounded and symmetric \citep{menendez1997jensen,lin1991divergence}. 
(2) Substituting $|\Delta H_t|$ with $\max(\Delta H_t, 0)$ incurs a $0.80$ drop. This confirms that tokens whose entropy decreases upon visual conditioning (where the image successfully disambiguates an uncertain prediction) serve as equally informative cues for visual dependency \citep{nam2025vague}.
(3) Replacing the multiplicative aggregation with additive fusion drops by $0.69$ points. The multiplicative form implements a soft probabilistic OR \citep{pearl2014probabilistic} that saturates when either signal is strong, while additive fusion under-learns non-additive interactions \citep{jayakumar2020multiplicative}. 
These results validate VEPO's bounded, direction-agnostic, and multiplicatively aggregated visual signals as principled selection choices.
More analysis about the visual perturbation is deferred to~\Cref{sec:appendix:Ablations on Visual Perturbation}.


\paragraph{Training dynamic and efficiency.}

\Cref{fig:figure2} (a) and (b) illustrate the training dynamics of VEPO and corresponding baselines. We observe that visual-focused RL is prone to premature convergence and optimization instability, whereas our method maintains stable training dynamics, yields higher rewards, thereby exhibits superior scalability \citep{DBLP:journals/corr/abs-2510-13786}.
\Cref{fig:figure2}(c) explains this gap mechanistically: at $k{=}0.2$, top-entropy selection misses ${\sim}40\%$ of vision-sensitive (Top-JSD / $|\Delta H_t|$) tokens. In contrast, VEPO reduces this to just $20\%$. By incorporating more visually sensitive tokens, VEPO enables the model to strike an optimal balance between perceptual grounding and reasoning informativeness. \Cref{fig:figure2} (d) reports per-step wall-clock time: by restricting gradient updates to top-$k$ tokens, VEPO is 
$\sim 10\%$ faster than Full-token GRPO and $\sim 16\%$ slower than Top-Entropy selection, which is a modest trade-off between computation and performance.

\paragraph{Qualitative Analysis.}
\Cref{fig:case_1} and~\Cref{fig:case_2} in~\Cref{sec:appendix:Qualitative Cases} illustrate the distinct token selection patterns between VEPO and the top-entropy mechanism. Both methods identify canonical reasoning anchors (``\texttt{Initialize}'', ``\texttt{condition}''), confirming a shared foundational baseline. Beyond this base, entropy-only selections drift to LaTeX scaffolding (``\texttt{\textbackslash frac}'', ``\texttt{\textbackslash boxed}'') and discourse fillers (``\texttt{follows}'', ``\texttt{we}''), 
where uncertainty stems from linguistic variability rather than visual dependency. In contrast, VEPO uniquely recovers visually-grounded content tokens, like variables and operators tied to the code (``\texttt{a}'', ``\texttt{+}'') and condition keywords referencing the flowchart (``\texttt{greater}'', ``\texttt{when}''). This demonstrates VEPO's ability to accurately redirect credit toward tokens that genuinely ground the textual reasoning in the visual input.

\section{Related Work}
\paragraph{Entropy Mechanism and Token Selection.}
In text-only reinforcement learning (RL), entropy is a well-formalized proxy for policy performance \citep{cui2025entropy}, where a small minority of high-entropy tokens dominates exploration and provides the primary gradient signals \citep{wang2026beyond,yu2026dapo,xi2026bapo}. Whether this principle transfers to the multimodal setting, however, remains underexplored \citep{li2025roleentropyvisualgrounding}: in visual reasoning, token uncertainty arises from heterogeneous sources spanning reasoning \citep{chung2025mllms}, visual perception and image noise \citep{fang2026enhancing}. Existing multimodal RL frameworks either conflate general high entropy with genuine visual dependency \citep{DoMLLMsReallySeeIt,chen2026ares}, or, as in VPPO \citep{SpotlightonToken} and PGPO \citep{DBLP:journals/corr/abs-2604-01840}, quantify per-token visual sensitivity via KL divergence between with and without-image distributions while treating visual signals in isolation. We bridge this gap with VEPO, which couples the vision-sensitivity with entropy to identify pivotal token-level decision points in visual reasoning.
\paragraph{Visual Perception Enhancement.}
Unlike the strong reasoning shown in text-only domains such as math and code \citep{deepseekr1,openai2026openaio1card,DBLP:journals/corr/abs-2512-17260}, MLLMs still lag on complex visual reasoning, a gap recently attributed to inadequate visual perception and grounding rather than reasoning itself \citep{SeeingbutNotBelieving,DoMLLMsReallySeeIt}. Current RL approaches fall into three lines: \emph{(i)} shaping visual attention via auxiliary rewards \citep{LookAgain,ReinforcedAttentionLearning,PAPO,DoMLLMsReallySeeIt}; \emph{(ii)} multi-stage workflows or curated visually-grounded data \citep{Look-Back,ghosal2026visrefvisualrefocusingthinking,wang2025vlrethinkerincentivizingselfreflectionvisionlanguage}; and \emph{(iii)} tool-augmented reasoning with external visual operations such as zooming or cropping \citep{zheng2026deepeyes,hong2026deepeyesv,su2025openthinkimglearningthinkimages}. These methods, however, either supervise all tokens uniformly or rely on external tool scaffolding. RAPT \citep{SeeingbutNotBelieving} further shows that deep-layer attention often localizes the correct evidence even when the final answer is wrong, suggesting that the bottleneck lies in precise attention \emph{supervision} rather than stronger attention. As our closest counterpart \citep{li2026rethinking}, PEPO reweights tokens using visual similarity—an indirect semantic proxy—whereas we comprehensively capture visual sensitivity via JSD and the entropy gap. 

\section{Conclusion}
In this paper, we revisit entropy-driven token selection in visual reasoning and show a counterintuitive failure mode: this mechanism sometimes overlooks vision-sensitive tokens that naturally exhibit low entropy but carry crucial visual evidence. Building on this diagnosis, we propose \textbf{VEPO}, a novel RL framework that explicitly couples visual sensitivity with token entropy through a multiplicative aggregation, redirecting gradient credit toward tokens that are simultaneously visually grounded and informative. Extensive experiments and ablations confirm VEPO's effectiveness. It achieves leading performance at comparable scales, surpassing the entropy-only baseline by +2.28 and +3.15 points on 7B and 3B models. We hope this work motivates further investigation into modality-aware credit assignment in multimodal reasoning.

\clearpage
\section*{Limitations}
Our work has a few limitations that warrant further investigation. First, the auxiliary forward pass required for visual perturbation introduces additional training overhead, though this could potentially be mitigated by advanced KV-cache reuse. Second, while Gaussian noise is an effective default counterfactual and we explicitly evaluated the no-image condition (\Cref{sec:appendix:Ablations on Visual Perturbation}) as an alternative perturbation, current computational limitations restricted our exploration of broader noise variants and sophisticated scheduling techniques like annealing \citep{liu2025noisyrolloutreinforcingvisualreasoning}. Lastly, while VEPO demonstrates strong performance and we offer an intuitive interpretation for its empirical effectiveness, providing a rigorous theoretical foundation for why visual-sensitive token selection improves multimodal RL optimization is left for future work.

\bibliography{custom}

\clearpage
\appendix

\section{Preliminary Experiment Setup} \label{appendix:Preliminary Experiments Setup}

\subsection{Setup for Preliminary Experiments in Section 2.1} \label{appendix:Preliminary Details:Preliminary Setup 2.1}
\paragraph{Setup.}
We compare three token selection strategies under the 80/20 rule \citep{wang2026beyond} in the standard GRPO\citep{shao2024deepseekmathpushinglimitsmathematical} setting:
(1)~\textbf{Full-token}, where all valid response tokens participate in the policy gradient update; 
(2)~\textbf{High-entropy}, which retains only the top-$k$ fraction of tokens ranked by policy entropy; and 
(3)~\textbf{Random}, which keeps each valid token independently with probability $p$ via Bernoulli sampling. 
For the high-entropy and random strategies, we sweep the retention ratio over $k \in \{0.1, 0.2, 0.4, 0.8\}$, and compare them against the full-token baseline ($k = 1.0$).

We conduct experiments on both LLMs and VLMs. For the LLM setting, we train Qwen2.5-VL-7B on a 54.4K mathematical reasoning dataset using the DAPO \citep{yu2026dapo} , with 8 H800 GPUs, a batch size of 256, 8 rollouts per prompt, and a maximum response length of 10,240 tokens.  We evaluate on AIME'24 . For the VLM setting, we follow the NoisyRollout setup for fair comparison and train Qwen2.5-VL-7B-Instruct on MMK12-mini (2.1K samples) using EasyR1\cite{zheng2025easyr1}, with 8 GPUs, a global batch size of 128, 12 rollouts per prompt, and a maximum response length of 2,048 tokens, while keeping the vision encoder frozen throughout training. We evaluate all settings on MathVerse \citep{mathverse}, MathVision \citep{wang2024measuringmultimodalmathematicalreasoning}, MathVista \citep{mathvista}, We-Math \citep{qiao2024wemathdoeslargemultimodal}, and HallusionBench \citep{Hallusionbench}. Both the LLM and VLM settings use a learning rate of $1 \times 10^{-6}$ and token-mean loss aggregation.

\subsection{Setup for Preliminary Experiments in Section 2.2} \label{appendix:Preliminary Details:Preliminary Setup 2.2}
\paragraph{Experiment Setup.}
To empirically investigate the relationship between policy entropy and visual dependency at the token level, we conduct an inference-time analysis using the same setup as Appendix~\ref{appendix:Preliminary Details:Preliminary Setup 2.1}. For each sample, we perform autoregressive generation with greedy decoding (temperature=0, top-$k$=1.0) and a maximum generation length of 256 tokens. 

\paragraph{Visual Dependence Metrics.}

During generation, we maintain two parallel decoding streams that share the same generated token sequence but differ in visual input: (1) a \emph{factual} stream conditioned on the original image $I$, and (2) a \emph{counterfactual} stream conditioned on a perturbed image $I'$. 
At each decoding step, we record three token-level signals: (i) the Jensen-Shannon divergence (JSD) between the factual distribution $P_t$ and the counterfactual distribution $Q_t$ over the full vocabulary, (ii) the Shannon entropy $H_t$ of the factual distribution, and (iii) the entropy gap
\[
\Delta H_t = H_t^{\mathrm{mask}} - H_t,
\]
which measures how degraded visual evidence changes model uncertainty. 

\paragraph{Analysis Protocol.}
We aim to quantify how well policy entropy serves as a proxy for visual dependency. For each token selection ratio $k$, we first construct an entropy-based selection set $A_{\mathrm{ent}}$ by taking the top-$k$ fraction of tokens ranked by entropy. We then compare it against two visual-dependency sets, denoted by $A_{\mathrm{vis}}$, which are defined as the top-$k$ tokens ranked by JSD and by $|\Delta H|$, respectively. We measure the overlap using recall, defined as
\[
\mathrm{Recall}
=
\frac{|A_{\mathrm{ent}} \cap A_{\mathrm{vis}}|}{|A_{\mathrm{vis}}|},
\]
where $|A_{\mathrm{ent}} \cap A_{\mathrm{vis}}|$ denotes the number of tokens selected by both methods, and $|A_{\mathrm{vis}}|$ is the size of the visual-dependency set.
\begin{table*}[!t]
\centering
\scriptsize
\setlength{\tabcolsep}{6pt}
\renewcommand{\arraystretch}{1.2}
\resizebox{\linewidth}{!}{
\begin{tabular}{l|ccccccccc}
\Xhline{1.2pt}
\rowcolor{CadetBlue!20}
\textbf{Dataset / Setting} & \textbf{Full} & \textbf{Top-0.1} & \textbf{Random-0.1} & \textbf{Top-0.2} & \textbf{Random-0.2} & \textbf{Top-0.4} & \textbf{Random-0.4} & \textbf{Top-0.8} & \textbf{Random-0.8} \\
\Xhline{1.2pt}

\hline
\rowcolor{gray!10} AIME'24 Turn 1 & 19.95  & 21.21 & 13.40 & 21.63 & 16.45 & 19.96 & 16.75 & 19.15 & 17.50 \\
\rowcolor{gray!10} AIME'24 Turn 2  & 20.62  & 19.75  & 16.30  & 22.49 & 15.62 & 20.58 & 15.80 & 20.10 & 17.90 \\
\rowcolor{gray!10} AIME'24 Turn 3    & 19.38 & 21.04 & 12.70  & 20.01 & 17.20 & 19.30 & 15.63 & 19.58 & 16.80 \\
\hline
\multicolumn{10}{c}{\textit{Avg.(Std.)}} \\
\hline
Avg.(Std.)       & 19.98(±0.26) & 20.67(±0.43) & 14.13(±2.39) & 21.38(±1.06) & 16.42(±0.42) & 19.95(±0.27) & 16.06(±0.24) & 19.61(±0.15)& 17.40(±0.20) \\
\Xhline{1.2pt}
\end{tabular}
}
\vspace{-0.3em}
\caption{Preliminary comparison of different token selection strategies in the LLM setting across three independent training runs. We compare full-token training, entropy-based Top-$k$\% selection, and random token selection under multiple retention ratios, and additionally report the mean and standard deviation across runs. The results clearly exhibit the 80/20 effect in language-only reasoning: entropy-based selection consistently outperforms random selection and achieves the best performance under sparse token retention.}
\label{tab:performance_simplified}
\end{table*}

\begin{table*}[!t]
\centering
\scriptsize
\setlength{\tabcolsep}{6pt}
\renewcommand{\arraystretch}{1.2}
\resizebox{\linewidth}{!}{
\begin{tabular}{l|ccccccccc}
\Xhline{1.2pt}
\rowcolor{CadetBlue!20}
\textbf{Dataset / Setting} & \textbf{Full} & \textbf{Top-0.1} & \textbf{Random-0.1} & \textbf{Top-0.2} & \textbf{Random-0.2} & \textbf{Top-0.4} & \textbf{Random-0.4} & \textbf{Top-0.8} & \textbf{Random-0.8} \\
\Xhline{1.2pt}

\multicolumn{10}{c}{\textit{Turn 1}} \\
\hline
\rowcolor{gray!10} HalluBench & 71.60  & 68.77 & 70.91 & 69.30 & 69.19 & 69.30 & 68.98 & 69.51 & 69.19 \\
\rowcolor{gray!10} MathVista  & 71.10  & 71.90  & 70.70  & 71.10 & 70.40 & 71.70 & 71.50 & 70.20 & 70.30 \\
\rowcolor{gray!10} We-Math    & 66.78 & 64.71 & 65.40  & 66.61 & 65.63 & 66.90 & 66.03 & 66.49 & 65.57 \\
\rowcolor{gray!10} MathVerse  & 47.16 & 46.90  & 45.96 & 46.68 & 47.79 & 47.66 & 46.52 & 47.21 & 47.23 \\
\rowcolor{gray!10} MathVision & 27.20  & 27.39 & 26.32 & 27.20 & 27.53 & 27.73 & 27.91 & 28.22 & 27.62 \\
\rowcolor{gray!10} Avg.       & 56.77 & 55.93 & 55.86 & 56.18 & 56.11 & 56.66 & 56.19 & 56.33 & 55.98  \\
\hline

\multicolumn{10}{c}{\textit{Turn 2}} \\
\hline
\rowcolor{gray!10}HalluBench & 68.24 & 70.03 & 70.98 & 70.14 & 69.61 & 69.61 & 70.35 & 69.40 & 69.30 \\
\rowcolor{gray!10}MathVista  & 69.90 & 69.50 & 69.07 & 70.10 & 69.80 & 71.00 & 70.30 & 70.70 & 70.20 \\
\rowcolor{gray!10}We-Math    & 68.10 & 65.00 & 65.92 & 66.09 & 65.98 & 67.82 & 64.48 & 65.23 & 67.30 \\
\rowcolor{gray!10}MathVerse  & 46.88 & 47.84 & 45.63 & 47.06 & 47.01 & 48.30 & 46.32 & 46.88 & 46.29 \\
\rowcolor{gray!10}MathVision & 28.66 & 28.30 & 27.77 & 26.84 & 27.93 & 27.48 & 27.51 & 27.76 & 27.35 \\
\rowcolor{gray!10}Avg.       & 56.36 & 56.13 & 55.87 & 56.05 & 56.07 & 56.84 & 55.79 & 55.99 & 56.08 \\
\hline

\multicolumn{10}{c}{\textit{Turn 3}} \\
\hline
\rowcolor{gray!10} HalluBench & 68.66 & 69.19 & 69.40 & 69.51 & 69.30 & 69.51 & 69.93 & 69.19 & 68.77 \\
\rowcolor{gray!10} MathVista  & 71.10 & 71.00 & 70.70 & 70.70 & 70.60 & 69.60 & 70.90 & 70.70 & 69.90 \\
\rowcolor{gray!10} We-Math    & 67.01 & 66.03 & 65.29 & 66.44 & 65.92 & 66.67 & 66.15 & 65.80 & 66.72 \\
\rowcolor{gray!10} MathVerse  & 47.26 & 47.18 & 46.02 & 47.92 & 46.73 & 46.83 & 46.83 & 47.31 & 47.39 \\
\rowcolor{gray!10} MathVision & 28.27 & 28.06 & 26.25 & 27.91 & 27.16 & 27.62 & 27.25 & 28.36 & 28.03 \\
\rowcolor{gray!10} Avg.       & 56.46 & 56.29 & 55.53 & 56.50 & 55.94 & 56.03 & 56.21 & 56.27 & 56.16 \\
\hline

\multicolumn{10}{c}{\textit{Avg.(Std.)}} \\
\hline
Avg.(Std.)       & 56.53(±0.03) & 56.12(±0.02) & 55.75(±0.02) & 56.24(±0.04) & 56.04(±0.01) & 56.51(±0.12) & 56.06(±0.04) & 56.20(±0.02)& 56.07(±0.01) \\

\Xhline{1.2pt}
\end{tabular}
}
\vspace{-0.3em}
\caption{Preliminary comparison of different token selection strategies in the VLM setting across three independent training runs. We compare full-token training, entropy-based Top-$k$\% selection, and random token selection under multiple retention ratios, and additionally report the mean and standard deviation across runs. Overall, entropy-based selection provides only limited and unstable gains in the VLM setting, remaining close to random selection and generally failing to consistently surpass the full-token baseline.}
\label{tab:performance_simplified_vlm}
\end{table*}

\begin{figure*}[H]
\centering
\begin{tcolorbox}[
    enhanced,
    width=0.96\linewidth,
    colback=gray!8,
    colframe=green!45!black,
    boxrule=0.6pt,
    arc=4mm,
    left=3mm,
    right=3mm,
    top=2mm,
    bottom=2mm,
    title={\Large\bfseries Generated Sample by VEPO-7B},
    colbacktitle=green!35!white,
    coltitle=white,
    fonttitle=\bfseries
]

\begin{minipage}[t]{0.32\linewidth}
\vspace{0pt}
\small
\textbf{Question.} The running result of the following program is \_\_\_.

\vspace{0.4em}
\includegraphics[width=\linewidth]{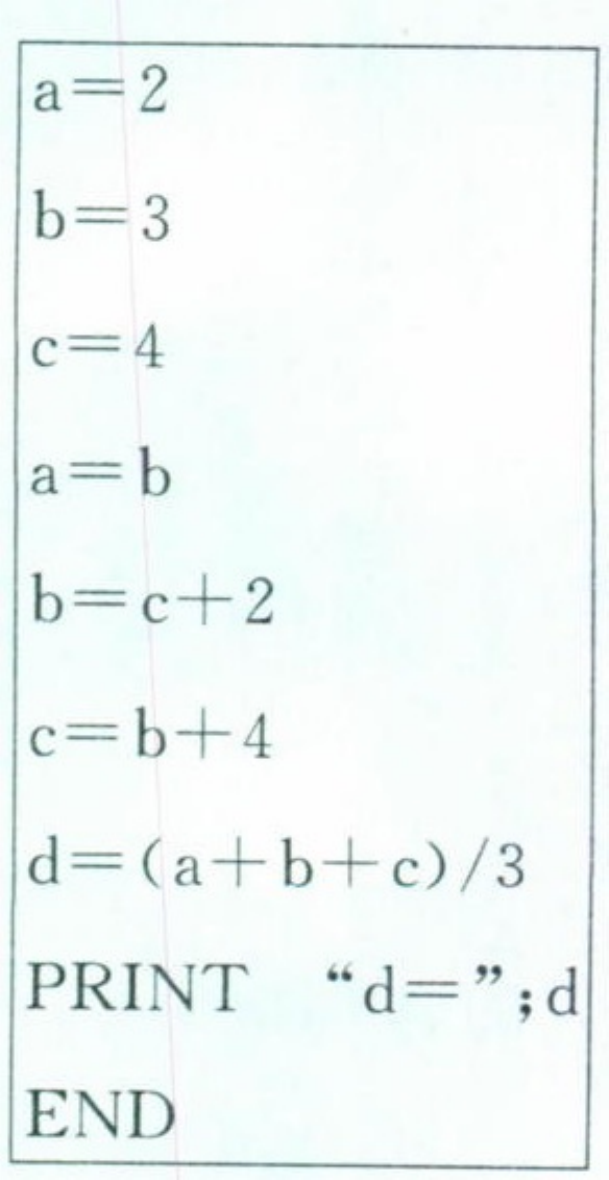}

\vspace{0.4em}
{\color{green!60!black}\textbf{Ground truth:} $\boxed{\frac{19}{3}}$}
\end{minipage}
\hfill
\begin{minipage}[t]{0.64\linewidth}
\vspace{0pt}
\small
\textbf{\texttt{<think>}}

The reasoning process is as follows:

\begin{enumerate}
    \item Initialize the variables:
    \[
    a = 2,\quad b = 3,\quad c = 4.
    \]

    \item Assign $b$ to $a$:
    \[
    a = b = 3.
    \]

    \item Assign $c + 2$ to $b$:
    \[
    b = c + 2 = 4 + 2 = 6.
    \]

    \item Assign $b + 4$ to $c$:
    \[
    c = b + 4 = 6 + 4 = 10.
    \]

    \item Compute $d$ as the average of $a$, $b$, and $c$:
    \[
    d = \frac{a+b+c}{3}
      = \frac{3+6+10}{3}
      = \frac{19}{3}.
    \]

    \item Therefore, the printed result is
    \[
    \frac{19}{3}.
    \]
\end{enumerate}

\textbf{\texttt{</think>}}

\vspace{0.5em}
The correct answer is {\color{green!60!black}\fbox{$\dfrac{19}{3}$}}.
\end{minipage}

\end{tcolorbox}
\caption{A qualitative case study on a visually grounded program-tracing problem. The model correctly tracks variable assignments and arithmetic updates from the program image, and outputs the correct final answer.}
\label{fig:case_program_math_2}
\end{figure*}
\section{More Detailed Results on Preliminary Experiments} \label{appendix:Preliminary Details} \label{appendix:Preliminary Details:More Preliminary Results.}
\Cref{tab:performance_simplified} and~\Cref{tab:performance_simplified_vlm} report preliminary comparisons of different token selection strategies across three independent runs in the LLM and VLM settings, respectively, including full-token training, entropy-based Top-$k$\% selection, and random selection.

We first examine the LLM setting. The results clearly support the 80/20 rule in language-only reasoning: entropy-based Top-$k$\% selection consistently outperforms Random-$k$\%, and the gains are most pronounced when $k$ is small. In particular, sparse selection performs best, and both Top-10\% and Top-20\% surpass the full-token baseline. This suggests that, in LLM reasoning, a small subset of high-entropy tokens indeed contributes disproportionately to exploration and optimization.

In contrast, the VLM setting exhibits a markedly different pattern. Here, entropy-based token selection provides only limited and unstable benefits: its performance is generally close to that of random selection and does not consistently outperform the full-token baseline. A representative phenomenon is that the best entropy-based variant remains only on par with, or slightly below, the full-token setting, while the gap between Top-$k$\% and Random-$k$\% remains small across different retention ratios. This indicates that, for multimodal reasoning, token entropy alone is not a reliable proxy for visually critical positions.

Overall, these preliminary results reveal a clear discrepancy between the two settings: while entropy-based token selection is effective in pure language reasoning and follows the 80/20 rule \citep{wang2026beyond}, its advantage becomes much weaker in visual reasoning, where token importance must additionally reflect visual dependency rather than uncertainty alone.
\section{Implementation  Details of VEPO}
\begin{table}[H]
\centering
\small
\setlength{\tabcolsep}{8pt}
\begin{tabular}{lc}
\toprule
\textbf{Hyperparameter} & \textbf{Value} \\
\midrule
Global batch size & 128 \\
Rollout batch size & 512 \\
Rollout number & 12 \\
Rollout temperature & 1.0 \\
Freeze vision tower & True \\
KL loss & False \\
Learning rate & $1.0 \times 10^{-6}$ \\
Weight decay & $1.0 \times 10^{-2}$ \\
Noisy augmentation & Gaussian \\
Noise step & 500 \\
\bottomrule
\end{tabular}
\caption{Main training configurations for VEPO.}
\label{tab:experiment_hyperparameters}
\end{table}

\begin{table*}[!t] 
\centering
\setlength{\tabcolsep}{6pt} 
\renewcommand{\arraystretch}{1.2}
\resizebox{\linewidth}{!}{
\begin{tabular}{l|cccccccc}
\Xhline{1.2pt}
\rowcolor{CadetBlue!20} 
\textbf{Model / Dataset} & \textbf{Geo3K} & \textbf{MMK12} & \textbf{HalluBench} & \textbf{MathVista} & \textbf{We-Math} & \textbf{MathVerse} & \textbf{MathVision} & \textbf{Avg.} \\
\Xhline{1.2pt}
\multicolumn{9}{c}{\textit{Visual-focused RL Fine-tuning Methods (3B)}} \\
\hline
\rowcolor{gray!10} NoisyRollout-3B \citep{liu2025noisyrolloutreinforcingvisualreasoning} & 42.26 & 58.12 & \textbf{63.85} & 62.20 & 61.78 & 40.94 & 25.07 & 50.56 \\
\rowcolor{gray!10} PAPO-DAPO-3B \citep{PAPO} & {42.76} & \uline{58.85} & \uline{62.36} & \uline{63.56} & 62.30 & 40.79 & 25.10 & \uline{50.82} \\
\rowcolor{gray!10} VPPO-3B \citep{SpotlightonToken} &  42.10 & 58.24 & 58.46 & \textbf{64.20} & \uline{62.64} & \uline{41.57} & 24.54 & 50.25 \\
\rowcolor{gray!10} R1-ShareVL \citep{yao2026rsharevl} & 39.93 & 56.26 & 61.41 & 60.30 & 61.26& 40.40 & 25.23 & 49.27 \\
\hline
\multicolumn{9}{c}{\textit{Our Method}} \\
\hline   
\rowcolor{gray!10} Qwen2.5-VL-3B-Instruct \citep{DBLP:journals/corr/abs-2502-13923} & 26.96 & 45.35 & 60.88 & 56.50 & 53.97 & 32.90 & 21.78 & 42.63 \\
\rowcolor{gray!10}
\hspace{2em}+ GRPO & \uline{43.43} & 57.37 & 59.52 & 62.16 & 61.15 & 40.95 & \uline{25.57} & 50.02 \\
\rowcolor{gray!10}
\hspace{2em}+ GRPO (Top 20\% high entropy)
& 38.28 & 56.81 & 59.20 & 61.40 & 59.66 & 37.74 & 24.18 & 48.18 \\
\rowcolor{gray!10}
\hspace{2em}+ GRPO (Top 40\% high entropy)
& 40.77 & 54.52 & 59.73 & 62.60 & 59.77 & 38.78 & 24.05 & 48.60 \\
\rowcolor{gray!10}
\hline
\rowcolor{gray!10}
\hspace{2em}+ \textbf{VEPO}
& \textbf{44.93} & \textbf{58.86} & 60.46 & 62.80 & \textbf{62.70} & \textbf{43.32} & \textbf{26.25} & \textbf{51.33} \\[-1.9pt]
\rowcolor{gray!10} 
\hspace{2em}${\triangle}$  vs. GRPO(Top 20\% high entropy)
 & {\textcolor{green(pigment)}{(+6.65)}}
 & {\textcolor{green(pigment)}{(+2.05)}}
 & {\textcolor{green(pigment)}{(+1.26)}}
 & {\textcolor{green(pigment)}{(+1.40)}}
 & {\textcolor{green(pigment)}{(+3.04)}}
 & {\textcolor{green(pigment)}{(+5.58)}}
 & {\textcolor{green(pigment)}{(+2.07)}}
 & {\textcolor{green(pigment)}{(+3.15)}} \\
\Xhline{1.2pt}
\end{tabular}
}
\vspace{-0.3em}
\caption{Additional results on the 3B model. Consistent with our main findings, VEPO (bottom) outperforms all corresponding baselines at this scale. Best scores are in \textbf{bold}, and the second-best are \uline{underlined}.}
\label{tab:performance_simplified_3B}
\vspace{-1em}
\end{table*}
\begin{table*}[!t]
\centering
\tiny
\quad
\setlength{\tabcolsep}{6pt}
\renewcommand{\arraystretch}{1.2}
\resizebox{\linewidth}{!}{
\begin{tabular}{l|cccccccc}
\Xhline{1.2pt}
\rowcolor{CadetBlue!20}
\textbf{$\alpha$ / Dataset} & \textbf{Geo3K} & \textbf{MMK12} & \textbf{HalluBench} & \textbf{MathVista} & \textbf{We-Math} & \textbf{MathVerse} & \textbf{MathVision} & \textbf{Avg.} \\
\Xhline{1.2pt}

\multicolumn{9}{c}{\textit{$k$ = 0.2}} \\
\hline
\rowcolor{gray!10}
$\alpha$ = 0.1  & 47.09 & 67.53 & 68.77 & 69.30 & 66.44 & 48.32 & 28.13 & 56.51 \\
\rowcolor{gray!10}
$\alpha$ = 0.3  & 48.25 & 67.78 & 68.24 & 71.30 & 67.24 & 48.43 & 27.20 & 56.92 \\
\rowcolor{gray!10}
$\alpha$ = 0.5  & 49.42 & 68.65 & 68.45 & 70.80 & 67.99 & 48.91 & 27.27 & 57.36 \\
\rowcolor{gray!10}
$\alpha$ = 0.7  & 51.58 & \textbf{69.64} & \textbf{71.71} & \textbf{72.00} & \textbf{69.54} & 48.93 & 28.31 & \textbf{58.82}\\
\rowcolor{gray!10}
$\alpha$ = 0.9  & 50.08 & 67.53 & 68.45 & 71.30 & 67.36 & \textbf{50.10} & 27.86 & 57.53 \\
\hline

\multicolumn{9}{c}{\textit{$k$ = 0.3}} \\
\hline
\rowcolor{gray!10}
$\alpha$ = 0.1 & 49.25 & 66.17 & 68.03 & 70.00 & 65.46 & 48.91 & 26.91 & 56.39 \\
\rowcolor{gray!10}
$\alpha$ = 0.3 & 50.75 & 68.65 & 70.70 & 71.70 & 68.10 & 49.82 & 28.22 & 58.28 \\
\rowcolor{gray!10}
$\alpha$ = 0.5 & 50.92 & 68.16 & 71.29 & 70.30 & 68.28 & 49.90 & 28.12 & 58.14 \\
\rowcolor{gray!10}
$\alpha$ = 0.7 & \textbf{52.08} & 68.40 & 69.51 & 71.50 & 68.16 & 48.55 & 27.70 & 57.99 \\
\rowcolor{gray!10}
$\alpha$ = 0.9 & 48.92 & 68.03 & 67.51 & 69.80 & 67.99 & 48.20 & 27.50 & 56.85 \\
\hline

\multicolumn{9}{c}{\textit{$k$ = 0.4}} \\
\hline
\rowcolor{gray!10}
$\alpha$ = 0.1 & 47.42 & 66.42 & 67.30 & 68.90 & 67.07 & 44.31 & \textbf{28.79} & 55.74 \\
\rowcolor{gray!10}
$\alpha$ = 0.3 & 48.25 & 66.91 & 69.30 & 64.50 & 67.36 & 48.15 & 25.93 & 55.77 \\
\rowcolor{gray!10}
$\alpha$ = 0.5 & 48.25 & 64.81 & 67.19 & 70.30 & 68.39 & 46.27 & 26.61 & 55.97 \\
\rowcolor{gray!10}
$\alpha$ = 0.7 & 51.25 & 68.77 & 68.56 & 69.60 & 67.76 & 47.66 & 28.06 & 57.38 \\
\rowcolor{gray!10}
$\alpha$ = 0.9 & 48.92 & 67.66 & 67.82 & 68.90 & 66.55 & 46.98 & 27.07 & 56.27 \\
\Xhline{1.2pt}
\end{tabular}
}
\vspace{-0.3em}
\caption{Detailed numerical results for the ablation on the token selection ratio $k$ and the balancing coefficient $\alpha$. The table reports the performance on each benchmark together with the overall average for every $(k,\alpha)$ configuration.}
\label{tab:ablation_hyperpara_appendix}

\end{table*}

\label{appendix:implementation details}
The main training config are summarized in Table~\ref{tab:experiment_hyperparameters}. We use a global batch size of 128 and a rollout batch size of 512, with 12 rollouts per prompt at a rollout temperature of 1.0. The policy model is optimized using a learning rate of $1.0 \times 10^{-6}$ and a weight decay of $1.0 \times 10^{-2}$. To improve efficiency and stability, we freeze the vision tower, disable KL loss, and adopt tensor parallelism with size 4 on 8 H800 GPUs. For noisy visual augmentation, we apply Gaussian perturbation at diffusion step 500 and use a sigmoid decay schedule with coefficient 30 and midpoint step 40.
For automatic evaluation, we use an external judge model (\texttt{gpt-4o-mini})  to parse model outputs and score final answers.
\paragraph{Dataset and Benchmark Details.}
For in-domain evaluation, we report results on the validation splits of Geometry3K \citep{geo3k} and MMK12 \citep{meng2025mm}, which measure performance within the training distribution. For out-of-domain evaluation, we use five public benchmarks to assess generalization beyond the training domains: four visual mathematical reasoning benchmarks, including MathVerse \citep{mathverse}, MathVision \citep{wang2024measuringmultimodalmathematicalreasoning}, MathVista \citep{mathvista}, and We-Math \citep{qiao2024wemathdoeslargemultimodal}, together with the visual perception benchmark HallusionBench \citep{Hallusionbench}.

\begin{table*}[!t]
\centering
\scriptsize
\setlength{\tabcolsep}{5pt}
\renewcommand{\arraystretch}{1.2}
\resizebox{\linewidth}{!}{
\begin{tabular}{l|cccccccc}
\Xhline{1.2pt}
\rowcolor{CadetBlue!20}
\textbf{Variant} & \textbf{Geo3K} & \textbf{MMK12} & \textbf{HalluBench} & \textbf{MathVista} & \textbf{We-Math} & \textbf{MathVerse} & \textbf{MathVision} & \textbf{Avg.} \\
\Xhline{1.2pt}
\rowcolor{gray!10}
Qwen2.5-7B-Instruct & 35.94 & 53.66 & 63.51 & 66.87 & 62.13 & 43.27 & 25.10 & 50.07 \\
\rowcolor{gray!10}
GRPO & 49.42 & 65.18 & 70.45 & 71.60 & 67.18 & 48.53 & 27.40 & 57.11 \\
\rowcolor{gray!10}
Ours w/ Masked Image & \textbf{52.58} & 68.65 & 71.19 & 71.60 & 68.65 & \textbf{49.21} & 27.93 & 58.54 \\
\rowcolor{gray!10}
Ours w/ Gaussian Noise & 51.58 & \textbf{69.64} &\textbf{ 71.71} & \textbf{72.00} & \textbf{69.54} & 48.93 & \textbf{28.31} & \textbf{58.82} \\
\Xhline{1.2pt}
\end{tabular}
\vspace{-0.3em}
}
\caption{Ablation on the role of visual perturbation in VEPO. We compare the base model, standard GRPO, our method with masked image input, and the full VEPO variant with Gaussian perturbation.}
\label{tab:ablation_visual_perturbation}
\end{table*}
\paragraph{Baseline.} \label{appendix:Baseline}
We compare our method against four groups of baselines: (1) closed-source multimodal models, including GPT-4o \citep{openai2024gpt4ocard}, o1 \citep{openai2026openaio1card}, Claude-3.7-Sonnet \citep{claude37}, and Gemini-2.0-Flash \citep{geminiteam2025gemini25}; (2) open-source large-scale vision-language models, including QVQ-72B-Preview \citep{qvq-72b-preview}, Qwen2.5-VL-32B \citep{DBLP:journals/corr/abs-2502-13923}, Qwen2.5-VL-72B \citep{DBLP:journals/corr/abs-2502-13923}, InternVL2.5-38B \citep{chen2025internvl}, and InternVL2.5-78B \citep{chen2025internvl}; (3) visual-focused RL fine-tuning methods, including \textbf{VPPO} \citep{SpotlightonToken}, \textbf{PAPO-DAPO} \citep{PAPO}, \textbf{R1-ShareVL} \citep{yao2026rsharevl}, and \textbf{NoisyRollout} \citep{liu2025noisyrolloutreinforcingvisualreasoning}, which are trained on the same 4.2K mixed training set for a fair 
comparison; and (4) two baselines that isolate the effect of 
entropy-based token selection: the top-entropy variant 
\textbf{80/20 Rule}~\citep{wang2026beyond} and the full-token 
\textbf{GRPO}~\citep{shao2024deepseekmathpushinglimitsmathematical} baseline.

\section{Results on Qwen2.5-VL-3B-Instruct} \label{appendix:Results on Qwen2.5-VL-3B-Instruct}
VEPO achieves the best average of \textbf{51.33}, surpassing the strongest visual-focused RL baseline PAPO-DAPO-3B (50.82, based on DAPO \citep{yu2026dapo}) and exceeding VPPO-3B by \textbf{+1.08}, also improving over the GRPO (Top 20\% high entropy) baseline by \textbf{+3.15}, while ranking first on six of seven benchmarks. 

\section{More Ablation Studies and Analysis} \label{appendix:More Ablation Studies}
\subsection{Detailed Results for Ablation on $\alpha$ and $k$}
\label{appendix:Ablation a and k}
The coefficient $\alpha$ balances the contribution of the distribution-shift signal JSD and the entropy-gap signal $|\Delta H|$, while $k$ controls the fraction of top-ranked tokens selected for policy updates in each response. We vary $\alpha$ in $\{0.1, 0.3, 0.5, 0.7, 0.9\}$ and $k$ in $\{0.2, 0.3, 0.4\}$ to examine the impact of different scoring preferences and selection granularities. The extreme cases $\alpha=0$ and $\alpha=1$ are excluded here, as they correspond to component-wise variants and are analyzed separately in our scoring component ablation. Table~\ref{tab:ablation_hyperpara_appendix} reports the detailed numerical results of the hyperparameter ablation on the token selection ratio $k$ and the balancing coefficient $\alpha$. These results complement the trends summarized in Figure~\ref{fig:ablation_alpha_k} in the main text.

\subsection{Ablations on Visual Perturbation}\label{sec:appendix:Ablations on Visual Perturbation}
We further analyze the role of visual perturbation by comparing two VEPO variants: one using masked image input and the other using Gaussian-noise perturbation. As shown in Table~\ref{tab:ablation_visual_perturbation}, the Gaussian variant achieves the best overall average performance (58.82 vs.\ 58.54) and outperforms the masked-image variant on five of the seven benchmarks, including MMK12, HalluBench, MathVista, We-Math, and MathVision. The gains are especially consistent on benchmarks that require stronger visual perception and fine-grained grounding, such as HalluBench and MathVision.

This pattern suggests that the form of perturbation plays an important role in token selection. Compared with masked image input, Gaussian noise perturbs the visual evidence while largely preserving the global image structure. As a result, the discrepancy between the original-image and perturbed-image predictions more faithfully reflects fine-grained visual dependency, rather than a broader degradation caused by an uninformative visual input. In this sense, Gaussian perturbation provides a more informative counterfactual signal for identifying visually dependent tokens.

Overall, this comparison supports the design choice in VEPO: the benefit comes not merely from contrasting two visual conditions, but from using a perturbation that preserves enough structure to reveal meaningful token-level visual sensitivity.

\section{Templates}
\label{appendix:templates}
\begin{tcolorbox}[
    colback=gray!5,
    colframe=black!70,
    boxrule=0.6pt,
    arc=2mm,
    width=0.95\linewidth,
    left=2mm,
    right=2mm,
    top=1mm,
    bottom=1mm
]
\small
\textbf{System Prompt:} \\
\texttt{You FIRST think about the reasoning process as an internal monologue and then provide the final answer. The reasoning process MUST BE enclosed within <think> </think> tags. The final answer MUST BE put in \textbackslash boxed\{\}.}
\vspace{0.4em}
\textbf{User Input:} \\
\texttt{\{question\}}
\end{tcolorbox}

\section{Theoretical Interpretation via Aleatoric--Epistemic Uncertainty Decomposition}
\label{appendix:theory}

In this appendix, we provide an information-theoretic interpretation of VEPO's joint vision-entropy token selection design (Eq.~\ref{eq:score_d} in the main paper). In particular, we show that the two visual signals: Jensen--Shannon Divergence (JSD, \citep{menendez1997jensen}) and the absolute entropy gap $|\Delta H_t|$ can be interpreted as capturing two complementary aspects of predictive change under visual perturbation. We further show that the multiplicative aggregation in $g_t$ admits a soft noisy-style OR \citep{pearl2014probabilistic}.

\subsection{Preliminaries}
\label{appendix:theory:prelim}

\paragraph{Setup.}
Recall from Eq.~(\ref{eq:forward}) that VEPO performs two forward passes at each token position $t$:
\begin{equation}
\begin{aligned}
P_t(\cdot) &= \pi_\theta(\cdot \mid q, I, o_{<t}), \\
Q_t(\cdot) &= \pi_\theta(\cdot \mid q, I', o_{<t}),
\end{aligned}
\label{eq:appendix-forward}
\end{equation}
where $I$ is the original image and $I'$ is its noise-perturbed counterpart. 

\paragraph{Visual condition.}
We introduce a binary variable $\Theta_v \in \{I, I'\}$ indicating whether the next-token prediction is conditioned on the original or perturbed image. Assuming a uniform prior $p(\Theta_v{=}I) = p(\Theta_v{=}I') = \tfrac{1}{2}$, the marginal predictive distribution of the next token $Y_t$ is
\begin{equation}
\begin{aligned}
M_t(\cdot)
&=
\mathbb{E}_{\Theta_v}\!\left[
\pi_\theta(\cdot \mid q, \Theta_v, o_{<t})
\right] \\
&=
\tfrac{1}{2}\bigl(P_t + Q_t\bigr),
\end{aligned}
\label{eq:appendix-mt}
\end{equation}
which coincides with the mixture distribution used in the definition of JSD (Eq.~\ref{eq:jsd}).

\paragraph{Uncertainty decomposition.}
Following the standard aleatoric--epistemic decomposition \citep{depeweg2018decomposition,houlsby2011bayesian,kendall2017uncertainties}, the predictive uncertainty of $Y_t$ under this visual-condition mixture can be decomposed as
\begin{equation}
\begin{aligned}
\underbrace{H(Y_t \mid q, o_{<t})}_{\text{Total Uncertainty (TU)}_t}
&=
\underbrace{I(Y_t; \Theta_v \mid q, o_{<t})}_{\text{Epistemic Component (EU)}_t} \\
&\quad +
\underbrace{\mathbb{E}_{\Theta_v}\!\left[
H(Y_t \mid q, \Theta_v, o_{<t})
\right]}_{\text{Aleatoric Uncertainty (AU)}_t}.
\end{aligned}
\label{eq:appendix-decomp}
\end{equation}
Here, the aleatoric term captures the residual uncertainty within each fixed visual condition, while the epistemic term captures the reducible disagreement across visual conditions.

\subsection{Proposition 1: JSD as a Visual-Condition Epistemic Component}
\label{appendix:theory:prop1}

\begin{proposition}[JSD as conditional mutual information]
\label{prop:jsd-eu}
Under the setup of~\S\ref{appendix:theory:prelim}, the Jensen--Shannon Divergence between $P_t$ and $Q_t$ is exactly the conditional mutual information between the next token and the visual condition:
\begin{equation}
\mathrm{JSD}(P_t \,\|\, Q_t)
=
I(Y_t; \Theta_v \mid q, o_{<t}).
\label{eq:appendix-prop1}
\end{equation}
\end{proposition}

\begin{proof}
By the definition of mutual information,
\begin{flalign}
&\hspace{1.1em} I(Y_t; \Theta_v \mid q, o_{<t}) && \notag\\
&\hspace{0.9em}= H(Y_t \mid q, o_{<t}) - H(Y_t \mid q, \Theta_v, o_{<t}) && \notag\\
&\hspace{0.9em}= H(M_t) - \mathbb{E}_{\Theta_v}\!\bigl[H(Y_t \mid q, \Theta_v, o_{<t})\bigr] && \notag\\
&\hspace{0.9em}= H(M_t) - \tfrac{1}{2}\bigl[H(P_t) + H(Q_t)\bigr]. &&
\label{eq:appendix-eu-expand}
\end{flalign}
Meanwhile, the Jensen--Shannon Divergence with uniform mixing weights \citep{lin1991divergence} satisfies
\begin{equation}
\begin{aligned}
\mathrm{JSD}(P_t \,\|\, Q_t)
&=
H(M_t) \\
&\quad - \tfrac{1}{2}\bigl[H(P_t) + H(Q_t)\bigr],
\end{aligned}
\label{eq:appendix-jsd-entropy}
\end{equation}
which matches Eq.~(\ref{eq:appendix-eu-expand}).
\end{proof}

\paragraph{Remark.}
Proposition~\ref{prop:jsd-eu} shows that JSD
is exactly the mutual information between the next token and the visual condition. In this sense, JSD quantifies how strongly the next-token prediction depends on the visual input.

\subsection{Proposition 2: $|\Delta H_t|$ as Condition-wise Aleatoric Change}
\label{appendix:theory:prop2}

\begin{proposition}[$|\Delta H_t|$ as aleatoric change]
\label{prop:dh-au}
The absolute entropy gap $|\Delta H_t|$ defined in Eq.~(\ref{eq:entropy_gap}) measures the magnitude of the change in conditional predictive entropy induced by visual perturbation:
\begin{equation}
\begin{aligned}
|\Delta H_t|
&=
\bigl|H(Y_t \mid q, \Theta_v{=}I', o_{<t}) \\
&\qquad - H(Y_t \mid q, \Theta_v{=}I, o_{<t})\bigr|.
\end{aligned}
\label{eq:appendix-prop2}
\end{equation}
\end{proposition}

\begin{proof}
Under each fixed visual condition, the conditional predictive entropy is
\begin{equation}
\begin{aligned}
H(Y_t \mid q, \Theta_v{=}I, o_{<t}) &= H(P_t), \\
H(Y_t \mid q, \Theta_v{=}I', o_{<t}) &= H(Q_t).
\end{aligned}
\end{equation}
Therefore, by Eq.~(\ref{eq:entropy_gap}),
\begin{equation}
\Delta H_t = H(Q_t) - H(P_t),
\end{equation}
and its absolute value equals the magnitude of the condition-wise entropy change.
\end{proof}

\paragraph{Remark.}
Unlike the aleatoric term in Eq.~(\ref{eq:appendix-decomp}), which averages uncertainty across visual conditions, $|\Delta H_t|$ measures how much the predictive entropy changes between them. A large $|\Delta H_t|$ therefore indicates that the token is sensitive to the visual input, whether the perturbation makes the model more confident or less confident.

\subsection{Proposition 3: Non-equivalence and Complementarity of JSD and $|\Delta H_t|$}
\label{appendix:theory:prop3}

\begin{proposition}[Non-equivalence and complementarity]
\label{prop:non_equiv}
The two visual signals $\mathrm{JSD}(P_t\|Q_t)$ and $|\Delta H_t| = |H(Q_t)-H(P_t)|$ capture complementary aspects of the predictive shift between $P_t$ and $Q_t$. In particular, neither signal determines the other alone:
(i) $|\Delta H_t|$ may vanish while $\mathrm{JSD}(P_t\|Q_t)$ remains strictly positive, and (ii) distinct pairs of distributions may share the same $|\Delta H_t|$ while exhibiting different JSD values.
\end{proposition}

\begin{proof}[Proof by construction]
We provide two examples on a 3-symbol vocabulary
$\mathcal{V}=\{v_1,v_2,v_3\}$.

\textbf{Case (a): Zero entropy gap but positive JSD.}
Let
\[
P_t = (1,0,0), \qquad Q_t = (0,1,0).
\]
Then both distributions are deterministic, so
\[
\begin{aligned}
H(P_t) &= H(Q_t) = 0, \\
|\Delta H_t| &= |H(Q_t)-H(P_t)| = 0.
\end{aligned}
\]
However, $P_t$ and $Q_t$ place all probability mass on different tokens, so
\[
\mathrm{JSD}(P_t\|Q_t)=\log 2 > 0.
\]

\textbf{Case (b): Same entropy gap but different JSD.}
Consider the two pairs:
\[
P_t^{(1)}=(1,0,0), \qquad
Q_t^{(1)}=\left(\tfrac12,\tfrac12,0\right),
\]
and
\[
P_t^{(2)}=(0,0,1), \qquad
Q_t^{(2)}=\left(\tfrac12,\tfrac12,0\right).
\]
For both pairs,
\begin{align*}
H(P_t^{(1)}) &= H(P_t^{(2)}) = 0, \\
H(Q_t^{(1)}) &= H(Q_t^{(2)}) = \log 2.
\end{align*}
Hence,
\[
|\Delta H_t^{(1)}| = |\Delta H_t^{(2)}| = \log 2.
\]
However, their Jensen--Shannon divergences differ:
\begin{align*}
\mathrm{JSD}(P_t^{(1)}\|Q_t^{(1)})
&=
\frac34 \log \frac43, \\
\mathrm{JSD}(P_t^{(2)}\|Q_t^{(2)})
&=
\log 2.
\end{align*}
Thus, the same entropy gap may correspond to substantially
different distributional disagreement.

Taken together, these examples show that $\mathrm{JSD}(P_t\|Q_t)$ and $|\Delta H_t|$ are non-equivalent and complementary: one may vanish while the other is positive, and even when the entropy gap is fixed, JSD can still vary substantially.
\end{proof}

\paragraph{Remark.}
Proposition~\ref{prop:non_equiv} clarifies the role of the two visual signals in VEPO. JSD captures disagreement between the predictive distributions under different visual conditions, whereas $|\Delta H_t|$ measures how much the corresponding predictive uncertainty changes. Because these signals encode different aspects of the predictive shift, relying on only one of them may fail to recover all visually sensitive tokens. This is consistent with the ablation results in \Cref{tab:ablation_components}, where removing either signal degrades performance.

\subsection{Proposition 4: $g_t$ as a soft noisy-OR–style Aggregation}
\label{appendix:theory:prop4}

\begin{proposition}[Soft noisy-OR aggregation]
\label{prop:noisy-or}
The joint visual-dependency score:
\begin{equation}
g_t = 1 - (1-\hat{j}_t)^\alpha (1-|\widehat{\Delta H}_t|)^{1-\alpha}
\label{eq:appendix-prop4}
\end{equation}
admits a noisy-OR–style interpretation as a soft union over two complementary visual-sensitivity cues.
\end{proposition}

\begin{proof}
For two independent binary events $A$ and $B$ with marginal probabilities $p_A$ and $p_B$, the classical noisy-OR–style takes the form:
\begin{equation}
P(A \cup B) = 1 - (1-p_A)(1-p_B).
\label{eq:appendix-pearl-or}
\end{equation}
VEPO's score in Eq.~(\ref{eq:appendix-prop4}) has the same complement-product structure, while introducing weights $\alpha$ and $1-\alpha$ to control the relative emphasis placed on JSD and $|\Delta H_t|$. As $\alpha \to 1$, $g_t \to \hat{j}_t$; as $\alpha \to 0$, $g_t \to |\widehat{\Delta H}_t|$. Thus, $g_t$ can be viewed as a soft union operator that prioritizes tokens that are salient under either visual cue, while still allowing one cue to dominate when appropriate.
\end{proof}

\paragraph{Remark.}
Proposition~\ref{prop:noisy-or} provides a probabilistic intuition for the multiplicative aggregation in VEPO. Unlike additive fusion, $\alpha\hat{j}_t + (1{-}\alpha)|\widehat{\Delta H}_t|$, the complement-product form saturates more naturally when either cue is strong, making it better suited for retaining tokens that spike along one visual-sensitivity axis but not necessarily the other. This interpretation is consistent with our empirical observation (Table~\ref{tab:ablation_components}) that replacing the multiplicative aggregation with additive fusion degrades performance.

\subsection{Putting It Together: The Full Scoring Function}
\label{appendix:theory:summary}

Combining Propositions~\ref{prop:jsd-eu}--\ref{prop:noisy-or},
VEPO's per-token score
\begin{flalign}
&\hspace{1.1em} c_t = g_t \cdot \hat{h}_t && \notag\\
&\hspace{0.9em}= \Bigl[
1 - (1-\hat{j}_t)^\alpha(1-|\widehat{\Delta H}_t|)^{1-\alpha}
\Bigr]\hat{h}_t . &&
\label{eq:appendix-final-score}
\end{flalign}
Here, the bracketed term acts as a soft union over the two complementary visual cues, while $\hat{h}_t$ measures token informativeness. This score admits a three-layer interpretation:
\begin{enumerate}
    \item[(i)] $\hat{j}_t$ captures visual-condition disagreement through conditional mutual information (Proposition~\ref{prop:jsd-eu});
    \item[(ii)] $|\widehat{\Delta H}_t|$ captures the magnitude of condition-wise predictive-entropy change (Proposition~\ref{prop:dh-au});
    \item[(iii)] $g_t$ combines these complementary cues through a noisy-OR-style soft union, while multiplication by $\hat{h}_t$ further prioritizes tokens that are also informative under the original (unperturbed) distribution.
\end{enumerate}
This interpretation helps explain why VEPO's joint vision-entropy token selection strategy is effective at identifying visually grounded and highly informative tokens.
\section{LLM Usage Statement}
We used Large Language Models (LLMs) as auxiliary tools to assist with the writing process. They
were used solely to polish the language and improve readability, with no influence over the research
design, experimental implementation or analysis. We conceived and executed all methodological
contributions, experiments, and conclusions independently.

\clearpage
\onecolumn

\section{Algorithm}
\label{appendix:Algorithm}
\Cref{alg:vepo} summarizes the VEPO pipeline. It first generates responses with standard rollout, then computes token-level visual-dependency signals by comparing model predictions under the original and perturbed images. Based on the resulting token scores, only the top-$k$ fraction of tokens is retained for policy gradient updates.
\vspace{0.5em}

\begin{algorithm}[h]
\caption{VEPO: Vision-Entropy Token Selection for Policy Optimization}
\label{alg:vepo}
\begin{algorithmic}[1]
\REQUIRE Policy model $\pi_\theta$; prompt set $\mathcal{D}$; original image $I$; noise step $\tau$; score weight $\alpha$; retention ratio $k$
\ENSURE Updated policy $\pi_\theta$

\FOR{each training iteration}
    \STATE \textcolor{gray}{\textit{// Stage 1: Rollout with original image}}
    \STATE Sample responses $\{o^{(k)}\}_{k=1}^{n} \sim \pi_\theta(\cdot \mid q, I)$
    \STATE Compute reward $r^{(k)}$ and group-relative advantage $A^{(k)}$

    \STATE \textcolor{gray}{\textit{// Stage 2: Counterfactual forward pass}}
    \STATE Construct perturbed image
    \STATE \hspace{1em} $I' \leftarrow \sqrt{\bar{\alpha}_\tau}\, I + \sqrt{1-\bar{\alpha}_\tau}\,\boldsymbol{\epsilon}, \quad \boldsymbol{\epsilon} \sim \mathcal{N}(0,\mathbf{I})$
    \FOR{each response $o = (o_1, \ldots, o_T)$}
        \FOR{each token position $t = 1, \ldots, T$}
            \STATE $P_t \leftarrow \pi_\theta(\cdot \mid q, I, o_{<t})$ \hfill \textcolor{gray}{// original image forward}
            \STATE $Q_t \leftarrow \pi_\theta(\cdot \mid q, I', o_{<t})$ \hfill \textcolor{gray}{// perturbed image forward}
        \ENDFOR
    \ENDFOR

    \STATE \textcolor{gray}{\textit{// Stage 3: Compute token-level visual dependency signals}}
    \FOR{each token position $t = 1, \ldots, T$}
        \STATE $M_t \leftarrow \frac{1}{2}(P_t + Q_t)$
        \STATE $\mathrm{JSD}_t \leftarrow \frac{1}{2}\mathrm{KL}(P_t \| M_t) + \frac{1}{2}\mathrm{KL}(Q_t \| M_t)$
        \STATE $H_t \leftarrow -\sum_{v} P_t(v) \log P_t(v)$
        \STATE $H_t^{\mathrm{mask}} \leftarrow -\sum_{v} Q_t(v) \log Q_t(v)$
        \STATE $\Delta H_t \leftarrow H_t^{\mathrm{mask}} - H_t$
    \ENDFOR

    \STATE \textcolor{gray}{\textit{// Stage 4: Score computation and token selection}}
    \STATE $\hat{j}_t,\; |\widehat{\Delta H}_t|,\; \hat{h}_t \leftarrow \text{MinMaxNorm}(\mathrm{JSD}_t,\; |\Delta H_t|,\; H_t)$ \hfill \textcolor{gray}{// per-response}
    \FOR{each token position $t = 1, \ldots, T$}
        \STATE $g_t \leftarrow 1 - (1-\hat{j}_t)^{\alpha} \cdot (1-|\widehat{\Delta H}_t|)^{1-\alpha}$ \hfill \textcolor{gray}{// joint visual-dependency signal}
        \STATE $c_t \leftarrow g_t \cdot \hat{h}_t$ \hfill \textcolor{gray}{// entropy-gated score}
    \ENDFOR
    \STATE $\mathcal{C}_t \leftarrow 1\!\left[c_t \in \mathrm{Top}\text{-}\lceil k \cdot T \rceil (\{c_\tau\}_{\tau=1}^{T})\right]$ \hfill \textcolor{gray}{// binary mask}

    \STATE \textcolor{gray}{\textit{// Stage 5: Masked policy gradient update}}
    \STATE $\mathcal{L}_{\mathrm{VEPO}} \leftarrow
    \dfrac{\sum_{t=1}^{T} \ell_t \cdot m_t \cdot \mathcal{C}_t}
    {\sum_{t=1}^{T} m_t \cdot \mathcal{C}_t}$
    \STATE Update $\theta \leftarrow \theta - \eta \nabla_\theta \mathcal{L}_{\mathrm{VEPO}}$
\ENDFOR
\end{algorithmic}
\end{algorithm}

\clearpage
\onecolumn
\section{Qualitative Cases}\label{sec:appendix:Qualitative Cases}
\begin{figure}[h]
\centering
\begin{tcolorbox}[
    enhanced,
    width=0.96\linewidth,
    colback=white,
    colframe=green!60!black,
    boxrule=0.7pt,
    arc=2.5mm,
    left=1.5mm,
    right=1.5mm,
    top=1.2mm,
    bottom=1.2mm,
    title={\Large\bfseries Qualitative Case 1: VEPO vs.\ Top-Entropy},
    colbacktitle=green!70!black,
    coltitle=white,
    fonttitle=\bfseries
]
\small

\begin{tcolorbox}[
    enhanced,
    colback=white,
    colframe=blue!70!black,
    boxrule=0.5pt,
    arc=1mm,
    left=1.2mm,
    right=1.2mm,
    top=1mm,
    bottom=1mm
]
\textbf{Question:} The running result of the following program is \_\_\_. \\
\textbf{Ground truth:} $\boxed{\frac{19}{3}}$

\vspace{0.4em}

\begin{minipage}[t]{0.23\linewidth}
\vspace{0pt}
\centering
\includegraphics[width=\linewidth]{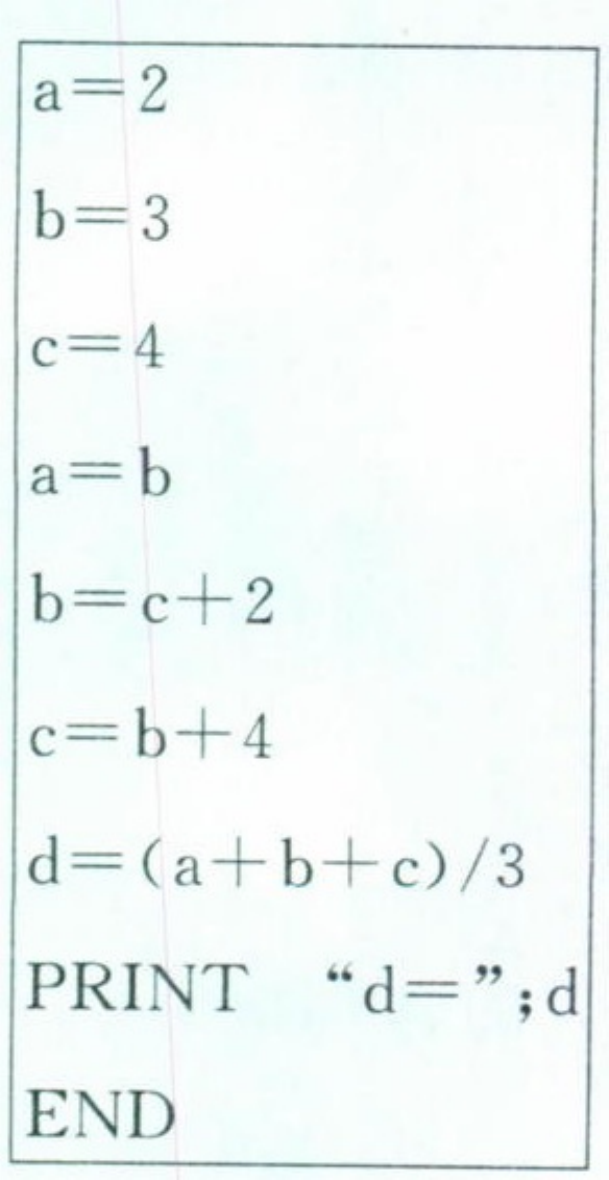}
\end{minipage}
\hfill
\begin{minipage}[t]{0.75\linewidth}
\vspace{0pt}
\small

\textbf{Legend.} 
\fcolorbox{blue!60!black}{blue!8}{\strut Both}\;
\fcolorbox{green!50!black}{green!10}{\strut VEPO-only}\;
\fcolorbox{red!65!black}{red!8}{\strut Entropy-only}

\vspace{0.4em}
\textbf{Generated response.}

\vspace{0.2em}
\texttt{<}\fcolorbox{red!65!black}{red!8}{\texttt{think}}\texttt{>}

\vspace{0.2em}
The
\fcolorbox{red!65!black}{red!8}{reasoning}
\fcolorbox{blue!60!black}{blue!8}{process}
is as follows:

\begin{enumerate}
\setlength{\itemsep}{0.15em}
\setlength{\parskip}{0pt}
\setlength{\parsep}{0pt}
    \item
    \fcolorbox{blue!60!black}{blue!8}{Initialize}
    the
    \fcolorbox{blue!60!black}{blue!8}{variables}:
    \( a = 2,\;
    \fcolorbox{red!65!black}{red!8}{\texttt{\textbackslash quad}}
    b = 3,\; c = 4 \)

    \item
    \fcolorbox{blue!60!black}{blue!8}{Assign}
    \(b\) to
    \fcolorbox{green!50!black}{green!10}{\(a\)}:
    \( a = b = 3 \)

    \item
    \fcolorbox{blue!60!black}{blue!8}{Assign}
    \fcolorbox{green!50!black}{green!10}{\(c\)}
    \fcolorbox{green!50!black}{green!10}{\(+\)}
    2 to
    \fcolorbox{green!50!black}{green!10}{\(b\)}:
    \( b = c + 2 = 4 + 2 = 6 \)

    \item
    \fcolorbox{blue!60!black}{blue!8}{Assign}
    \fcolorbox{green!50!black}{green!10}{\(b\)}
    \fcolorbox{green!50!black}{green!10}{\(+\)}
    4 to
    \fcolorbox{green!50!black}{green!10}{\(c\)}:
    \( c = b + 4 = 6 + 4 = 10 \)

    \item
    \fcolorbox{blue!60!black}{blue!8}{Calculate}
    \(d\)
    \fcolorbox{red!65!black}{red!8}{)}
    \fcolorbox{green!50!black}{green!10}{as}
    \fcolorbox{green!50!black}{green!10}{the}
    \fcolorbox{green!50!black}{green!10}{average}
    of \(a\), \(b\), and \(c\):
    \( d =
    \fcolorbox{red!65!black}{red!8}{\texttt{\textbackslash frac}}
    \{a+b+c\}\{3\}
    = \frac{3+6+10}{3}
    = \frac{19}{3} \)

    \item
    \fcolorbox{blue!60!black}{blue!8}{Print}
    \fcolorbox{green!50!black}{green!10}{the}
    value of \(d\):
    \( d = \frac{19}{3} \)
\end{enumerate}

\vspace{-0.3em}
\texttt{</}\fcolorbox{red!65!black}{red!8}{\texttt{think}}\texttt{>}
\quad
\texttt{\textbackslash}\fcolorbox{red!65!black}{red!8}{\texttt{boxed}}\texttt{\{$\frac{19}{3}$\}}
\end{minipage}
\end{tcolorbox}

\vspace{0.5em}

\begin{tcolorbox}[
    enhanced,
    colback=white,
    colframe=blue!70!black,
    boxrule=0.5pt,
    arc=1mm,
    left=1.2mm,
    right=1.2mm,
    top=1mm,
    bottom=1mm
]
\textbf{Representative Token Selections.}

\vspace{0.35em}
\textbf{Both:}
\fcolorbox{blue!60!black}{blue!8}{process}\;
\fcolorbox{blue!60!black}{blue!8}{Initialize}\;
\fcolorbox{blue!60!black}{blue!8}{variables}\;
\fcolorbox{blue!60!black}{blue!8}{Assign}\;
\fcolorbox{blue!60!black}{blue!8}{Assign}\;
\fcolorbox{blue!60!black}{blue!8}{Calculate}\;
\fcolorbox{blue!60!black}{blue!8}{Print}

\vspace{0.35em}
\textbf{VEPO-only:} 
\fcolorbox{green!50!black}{green!10}{a}\;
\fcolorbox{green!50!black}{green!10}{c}\;
\fcolorbox{green!50!black}{green!10}{+}\;
\fcolorbox{green!50!black}{green!10}{b}\;
\fcolorbox{green!50!black}{green!10}{b}\;
\fcolorbox{green!50!black}{green!10}{+}\;
\fcolorbox{green!50!black}{green!10}{c}\;
\fcolorbox{green!50!black}{green!10}{as}\;
\fcolorbox{green!50!black}{green!10}{the}\;
\fcolorbox{green!50!black}{green!10}{the}

\vspace{0.35em}
\textbf{Entropy-only:}
\fcolorbox{red!65!black}{red!8}{think}\;
\fcolorbox{red!65!black}{red!8}{reasoning}\;
\fcolorbox{red!65!black}{red!8}{:\textbackslash n\textbackslash n}\;
\fcolorbox{red!65!black}{red!8}{quad}\;
\fcolorbox{red!65!black}{red!8}{):\textbackslash n}\;
\fcolorbox{red!65!black}{red!8}{[\textbackslash n}\;
\fcolorbox{red!65!black}{red!8}{]\textbackslash n\textbackslash n}\;
\fcolorbox{red!65!black}{red!8}{)}\;
\fcolorbox{red!65!black}{red!8}{frac}\;
\fcolorbox{red!65!black}{red!8}{boxed}

\vspace{0.45em}

\end{tcolorbox}

\end{tcolorbox}
\caption{A qualitative comparison between VEPO and Top-Entropy on a visually grounded program-tracing example. The upper panel shows the program image and generated response, while the lower panel groups representative token selections into shared, VEPO-only, and entropy-only categories.}
\label{fig:case_1}
\end{figure}

\newcommand{\tokvepo}[1]{\fcolorbox{green!50!black}{green!10}{\strut\texttt{\detokenize{#1}}}}
\newcommand{\tokboth}[1]{\fcolorbox{blue!60!black}{blue!8}{\strut\texttt{\detokenize{#1}}}}
\newcommand{\tokent}[1]{\fcolorbox{red!65!black}{red!8}{\strut\texttt{\detokenize{#1}}}}
\begin{figure}[t]
\centering
\begin{tcolorbox}[
    enhanced,
    width=0.96\linewidth,
    colback=white,
    colframe=green!60!black,
    boxrule=0.7pt,
    arc=2.5mm,
    left=1.5mm,
    right=1.5mm,
    top=1.2mm,
    bottom=1.2mm,
    title={\Large\bfseries Qualitative Case 2: VEPO vs.\ Top-Entropy},
    colbacktitle=green!70!black,
    coltitle=white,
    fonttitle=\bfseries
]
\small

\begin{tcolorbox}[
    enhanced,
    colback=white,
    colframe=blue!70!black,
    boxrule=0.5pt,
    arc=1mm,
    left=1.2mm,
    right=1.2mm,
    top=1mm,
    bottom=1mm
]
\textbf{Question:} As shown in the figure, according to the given program, if the input is $x=\sqrt{3}$, then the output result is \_\_\_. \\
\textbf{Ground truth:} $\boxed{2}$

\vspace{0.4em}

\begin{minipage}[t]{0.23\linewidth}
\vspace{0pt}
\centering
\includegraphics[width=\linewidth]{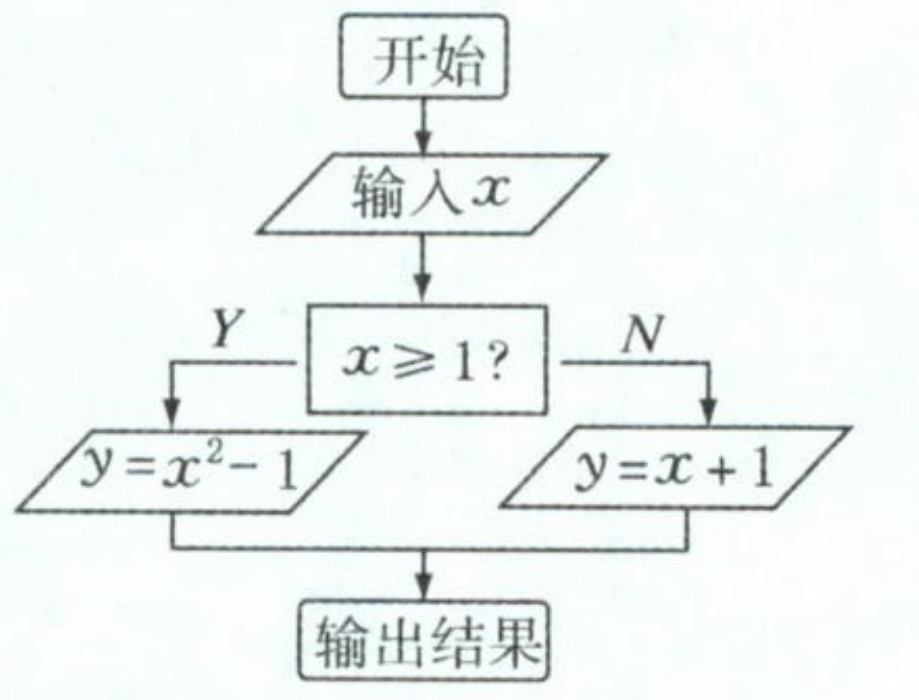}
\end{minipage}
\hfill
\begin{minipage}[t]{0.75\linewidth}
\vspace{0pt}
\small

\textbf{Legend.}
\tokboth{Both}\;
\tokvepo{VEPO-only}\;
\tokent{Entropy-only}

\vspace{0.4em}
\textbf{Generated response.}

\vspace{0.2em}
\texttt{<think>}

\vspace{0.2em}
The \tokboth{reasoning} \tokboth{process} is as \tokent{follows}:

Since the \tokboth{input} is \(x=\sqrt{3}\), we \tokboth{check} the
\tokboth{condition} \(x \ge \tokent{q} 1\).
Here, \(\sqrt{3} \approx 1.732\), \tokent{which} is
\tokvepo{greater} \tokvepo{than} 1.
Therefore, the \tokboth{condition} \(x \geq 1\) is true.

According to the \tokboth{flowchart}, \tokvepo{when} \(x \geq 1\),
the \tokboth{value} of \(y\) is \tokvepo{given} by \(y = x^2 - 1\).
So, \tokent{we} substitute \(x=\sqrt{3}\) into the \tokvepo{equation}:

\[
y = (\sqrt{3})^2 - 1 = 3 - 1 = 2
\]

Thus, the \tokboth{output} \tokboth{result} is 2.\texttt{</}\tokent{think}\texttt{>}
\quad
\texttt{\detokenize{\boxed{2}}}

\vspace{0.2em}
\end{minipage}
\end{tcolorbox}

\vspace{0.5em}

\begin{tcolorbox}[
    enhanced,
    colback=white,
    colframe=blue!70!black,
    boxrule=0.5pt,
    arc=1mm,
    left=1.2mm,
    right=1.2mm,
    top=1mm,
    bottom=1mm
]
\textbf{Representative Token Selections.}

\vspace{0.35em}
\textbf{Both:}
\tokboth{reasoning}\;
\tokboth{process}\;
\tokboth{input}\;
\tokboth{check}\;
\tokboth{condition}\;
\tokboth{flowchart}\;
\tokboth{value}\;
\tokboth{output}\;
\tokboth{result}

\vspace{0.35em}
\textbf{VEPO-only:}
\tokvepo{greater}\;
\tokvepo{than}\;
\tokvepo{when}\;
\tokvepo{given}\;
\tokvepo{equation}

\vspace{0.35em}
\textbf{Entropy-only:}
\tokent{follows}\;
\tokent{q}\;
\tokent{which}\;
\tokent{we}\;
\tokent{think}

\end{tcolorbox}

\end{tcolorbox}
\caption{Another qualitative comparison between VEPO and Top-Entropy on a visually grounded program-tracing example. }
\label{fig:case_2}
\end{figure}

\end{document}